\documentclass[compsoc,conference,a4paper,10pt]{IEEEtran}
\IEEEoverridecommandlockouts
\usepackage{cite}
\usepackage{amsmath,amssymb,amsfonts}
\usepackage{graphicx}
\usepackage{textcomp}
\usepackage{bmpsize}
\usepackage{xcolor}
\usepackage{lipsum}
\usepackage[colorlinks=true]{hyperref}

\usepackage{bm}
\usepackage{amsthm}
\usepackage{algorithm}
\usepackage{algpseudocode}

\usepackage{graphicx}
\usepackage{subcaption}

\usepackage{array}
\usepackage{booktabs}
\usepackage{caption}
\usepackage{multirow}
\usepackage{threeparttable}
\usepackage{adjustbox}

\usepackage{mdframed}
\usepackage{tcolorbox}

\newtheorem{lemma}{Lemma}
\newtheorem{theorem}{Theorem}
\newtheorem{assumption}{Assumption}

\def\BibTeX{{\rm B\kern-.05em{\sc i\kern-.025em b}\kern-.08em T\kern-.1667em\lower.7ex\hbox{E}\kern-.125emX}}

\newenvironment{myframed}[1][]{
  \begin{mdframed}[backgroundcolor=blue!10!white, linecolor=black!75, linewidth=0.6pt, roundcorner=5pt]
    #1
}{
  \end{mdframed}
}

\begin{document}

\title{GDBR: Label Recovery Attack Against Partial Gradient Encryption\\ in Federated Learning}

\author{
  \IEEEauthorblockN{Rui Zhang$^\dagger$}
  \thanks{$^\dagger$\,This work was completed while the author was pursuing his Ph.D. at The Hong Kong Polytechnic University.}
  \IEEEauthorblockA{
    \textit{HKU Musketeers Foundation Institute of Data Science} \\
    \textit{The University of Hong Kong} \\
    Hong Kong SAR, China \\
    csrzhang@hku.hk}
  \and
  \IEEEauthorblockN{Ka-Ho Chow$^*$}
  \thanks{$^*$\,Corresponding author.}
  \IEEEauthorblockA{
    \textit{School of Computing and Data Science} \\
    \textit{The University of Hong Kong} \\
    Hong Kong SAR, China \\
    kachow@cs.hku.hk}
}

\maketitle

\begin{abstract}
  The increasing demand for data privacy, alongside the benefits of aggregating data from networked devices, has catalyzed the emergence of federated learning (FL). In FL, clients jointly train a global model by sharing gradients computed over private data. While this paradigm eliminates the need to exchange raw data, inference attacks can still be launched to extract sensitive information from gradients. To this end, partial gradient encryption has emerged as a promising design for balancing privacy and efficiency in practical FL systems, as encrypting only the classification-head gradients is believed to prevent known inference attacks while avoiding the high computational cost of encrypting the entire model. However, this design provides a false sense of privacy. By proposing GDBR, we show that sharing even a single unencrypted layer of gradients can lead to serious privacy leakage. GDBR is the first attack capable of high-fidelity label recovery with partial access to the gradients. It exploits a vulnerability in a commonly used neural building block, constructs a gradient bridge from the unencrypted layer to the final output layer, and approximates the logits information for accurate inference of private labels. These inferred labels not only reveal sensitive information about a client's private dataset but also serve as a prerequisite for many downstream attacks, such as data reconstruction and membership inference. GDBR brings these threats squarely into scope for FL systems employing partial encryption. In addition to theoretical analysis, extensive experiments demonstrate the severity of the problem across a wide variety of datasets and model architectures, including convolutional and transformer-based networks. Overall, our findings challenge the widespread assumption that encrypting only the output layer suffices for privacy protection. Our source code is available at: \url{https://github.com/csrzhang/GDBR}.
\end{abstract}


\section{Introduction}

\begin{figure}[htbp]
  \centering
  \includegraphics[width=0.46\textwidth]{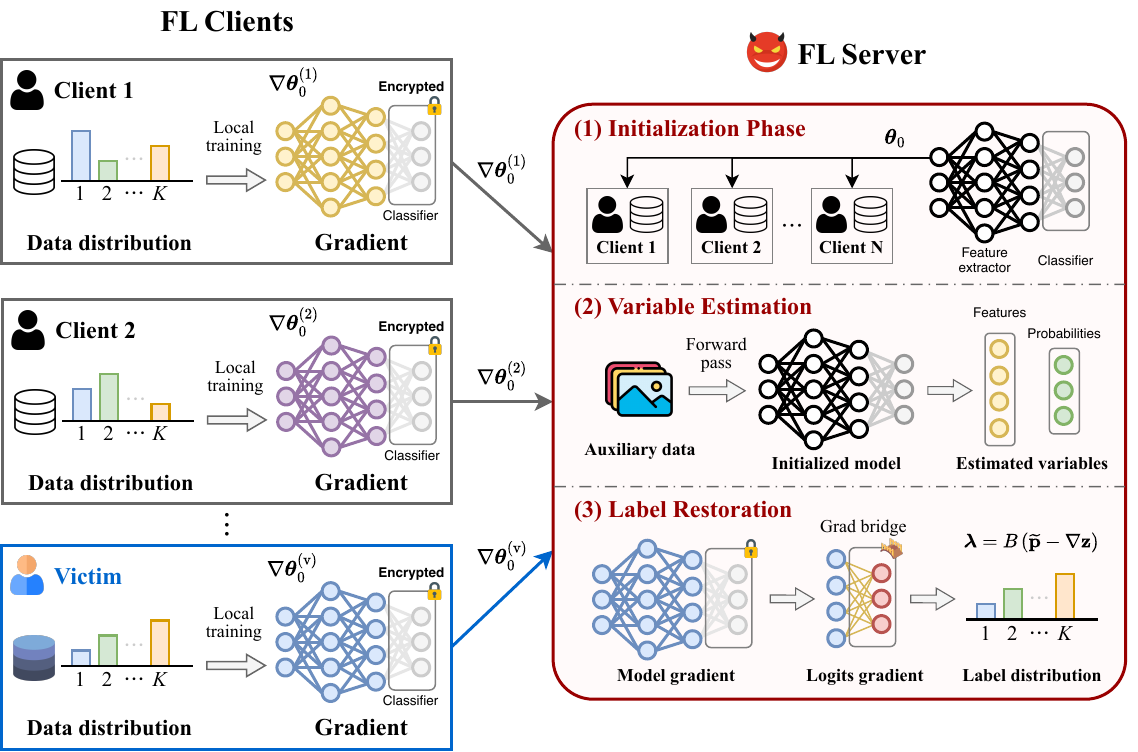}
  \caption{Overview of the GDBR attack. There are $N$ clients participating in FL, and each holds a private dataset with its own data distribution. The FL server is a semi-honest adversary and attempts to infer the label distribution of a victim client by exploiting the partially encrypted gradients. The GDBR attack proceeds as follows: (1) The FL server initializes and distributes the global model to the clients. (2) The adversary estimates the input features and output probabilities of the classification head using auxiliary information. (3) The adversary builds a \emph{gradient bridge} to infer the gradients with respect to the output logits from the accessible gradients, and then recovers the victim's label distribution using the derived formula.}
  \label{fig:overview}
\end{figure}

Federated learning (FL) \cite{konevcny2016federated,mcmahan2017communication,bonawitz2019towards} has emerged as a promising privacy-preserving paradigm that enables multiple clients to collaboratively train a machine learning model without directly exposing their raw data. In each communication round, the server distributes the current global model to the participating clients. Each client then computes gradients on their private data and sends the local gradients back to the server, which incorporates these updates into the global model. This iterative process continues until the model reaches satisfactory performance. Although exchanging gradients instead of raw data makes FL appear privacy-preserving, it has become widely recognized that sensitive information can still be inferred from gradients \cite{wei2020framework,zhang2023survey}. Among these threats, label recovery is particularly severe because it is not only a direct privacy leakage but also a prerequisite for many other attacks.

Label recovery aims to reconstruct the private labels directly from the shared gradients. Existing methods leverage the gradients of the last fully-connected (FC) layer to infer labels \cite{zhao2020idlg,dang2021revealing,yin2021see,geng2021towards,wainakh2022user,ma2023instancewise}. With these labels, gradient inversion attacks \cite{zhao2020idlg,dang2021revealing} can reconstruct high-fidelity training samples for the target batch, user-level membership inference attacks \cite{wang2024graddiff} can identify whether specific gradients originate from a particular client's data, and poisoning attacks \cite{mukhtiar2025fairness} can compromise the fairness of the global model by injecting bias into the training process.

Secure FL has been considered a mitigation for the privacy risks associated with label leakage. Since the FL server may extract sensitive information from gradients, clients can use homomorphic encryption to encrypt their gradients and transmit the ciphertexts to the server for aggregation \cite{zhang2020batchcrypt,fang2021privacy}. While this prevents the server from launching analytical or optimization-based attacks using the shared gradients, the computational cost is prohibitively high.
To balance privacy and efficiency, partial encryption has been employed to selectively encrypt only the model's output layer, where label-sensitive gradients reside. By sharing the remaining gradients in plaintext, this approach significantly reduces computational overhead while allowing the number of encrypted layers to be dynamically adjusted depending on the desired privacy-utility trade-off.

However, in this work we dismantle this widespread belief and reveal a severe vulnerability: even when the output layer's gradients are completely inaccessible, the shared gradients of an intermediate layer still leak precise private label information through inter-layer correlations. A single layer's gradient is sufficient for high-fidelity label recovery. To demonstrate the false sense of privacy offered by FL with partial encryption, we propose GDBR, the first attack capable of high-fidelity restoration of private labels under such challenging settings.

We observe that the input features and output probabilities have similar distributions across different samples, where each dimension follows an approximately Gaussian distribution. Thus, GDBR estimates the input features and output probabilities using auxiliary information, such as public datasets or fully generated dummy data. Based on this, GDBR constructs a gradient bridge that systematically propagates information across the model to approximate the missing logit gradients. From these approximated logit gradients, GDBR directly recovers the batch labels based on our theoretical findings. Fig.~\ref{fig:overview} illustrates the workflow of GDBR attack.

Our findings invalidate the core security assumption behind the partial encryption design and reveal a universal vulnerability in deep networks. As deep networks proliferate, similar inter-layer bridge vulnerabilities are poised to emerge ubiquitously, compelling the community to devise next-generation defenses that rigorously balance privacy and efficiency.

Our main contributions are summarized as follows:
\begin{itemize}
	\item We propose GDBR, the first attack achieving high-fidelity label reconstruction in federated learning even when the gradients of the head layer of the classifier are completely missing.
	\item We uncover and formally characterize the gradient bridge phenomenon: inter-layer correlations that inevitably leak logit-gradient information into earlier layers, providing fundamental new insights into gradient privacy in deep networks.
	\item We highlight the inadequacy of existing defenses, underscoring the need for more robust privacy-preserving techniques in FL systems.
\end{itemize}

Extensive experiments on diverse datasets, model architectures, and batch sizes (up to 512) demonstrate that GDBR achieves strong label recovery accuracies, even outperforming existing label attacks that have access to final FC layer's gradients in some settings. We further evaluate the robustness of GDBR under class distribution skew, randomly generated auxiliary data, additional defense mechanisms on the shared feature gradients, and transformer-based neural architectures. The results show that GDBR maintains strong performance in all scenarios. Overall, GDBR invalidates the core security assumption of the prevalent partial encryption design and urge the community to develop next-generation defenses that rigorously account for inter-layer leakage in deep networks.

\section{Related Work}

In this section, we introduce the related works on gradient inversion attacks, analytical label recovery attacks, and existing encryption-based defenses in FL.

\subsection{Gradient Inversion Attacks}

Gradient inversion attacks (GIAs) \cite{zhang2023survey} seek to reconstruct private training data by exploiting the exchanged gradients in FL. The clients share model updates instead of raw data to protect privacy; however, gradients remain susceptible to sensitive information leakage, making them attractive targets for adversarial exploitation.

Zhu et al.~\cite{zhu2019deep} propose the {DLG} attack, which iteratively optimizes dummy inputs and labels to minimize the gap between generated and ground-truth gradients. Geiping et al.~\cite{geiping2020inverting} enhance DLG by adopting cosine similarity as the loss function and adding total variation regularization, improving the fidelity of reconstructed images. Yin et al.~\cite{yin2021see} introduce group consistency regularization to better preserve object locations. Zhu and Blaschko~\cite{zhu2021rgap} present a recursive attack that recovers layer-wise features by solving linear systems. 

Beyond computer vision, similar inversion techniques have been applied to natural language processing~\cite{gupta2022recovering,balunovic2022lamp} and speech recognition~\cite{dang2022method}, indicating that gradient leakage poses a broad privacy risk across modalities.

\subsection{Analytical Label Distribution Recovery}

Analytical label recovery attacks aim to infer private label information from shared gradients in FL, often without reconstructing the input data itself. 
Since the gradients of the final fully-connected (FC) layer are highly correlated with the one-hot encoded labels, adversaries can exploit this property to recover label distributions.

Zhao et al.~\cite{zhao2020idlg} introduce {iDLG}, which leverages the relationship between the last layer's gradients and the label encoding to directly recover labels. 
Building on this, Yin et al.~\cite{yin2021see} propose {GI}, identifying the rows with minimum values in the final FC layer's gradient as the target class. 
Dang et al.~\cite{dang2021revealing} present {RLG}, which infers the target class by analyzing the columns of the gradient with respect to the output layer's weights. 
Geng et al.~\cite{geng2021towards} develop {ZLG}, using auxiliary data to estimate posterior probabilities and extract last-layer features for classwise label restoration. 
Wainakh et al.~\cite{wainakh2022user} design {LLG}, combining both the direction and magnitude of gradients to compute an overall impact factor and classwise offset, enabling sequential label recovery. 
Ma et al.~\cite{ma2023instancewise} propose {iLRG}, which derives embedding features and softmax probabilities from the output layer and solves linear systems to recover batch-wise labels.

Despite their differences, all these attacks rely fundamentally on access to the gradients of the final FC layer. 
Notably, iDLG, GI, and RLG are restricted to recovering non-repeating labels in the training batch, while iLRG requires additional gradients with respect to the biases for accurate feature restoration.



\subsection{Existing Defenses Against Label Recovery}

Recent studies have explored the use of cryptographic techniques to counter gradient-based batch label recovery in collaborative learning. A central idea in these works is that the parameters most closely tied to the output layer, typically the classification head or the last few layers, can be protected more efficiently through \emph{selective} or \emph{partial} encryption rather than encrypting the entire model.

FedML-HE \cite{jin2023fedml} employs homomorphic encryption on a sensitivity-identified subset of model parameters, demonstrating that protecting only the components most exposed to gradient-level leakage can significantly reduce computation while maintaining privacy. 
MaskCrypt \cite{hu2024maskcrypt} introduces a client-negotiated masking protocol, enabling homomorphic encryption of chosen parameter subsets such as final-layer weights, which typically influence gradient observability the most. 
SHE-LoRA \cite{liu2025shelora} adopts a sensitivity-ranking mechanism to determine which parameters should be encrypted, naturally covering scenarios where the last-layer parameters require stronger protection. 
SenseCrypt \cite{li2025sensecrypt} further refines this direction by adaptively selecting which parameters to encrypt based on sensitivity clustering, where more sensitive components that encode label information are protected with homomorphic encryption.

Collectively, these works show a converging trend: protecting a chosen portion of the model, that is, the classification head, offers an effective defense strategy in settings where gradients may inadvertently reveal label information.

\section{Preliminaries}

In this section, we present the preliminaries of GDBR, including the notations used throughout this paper, the gradient correlations in typical neural network layers, the correspondence between model's logits and one-hot labels, and the assumed threat model in FL.

\subsection{Notations}

In Table \ref{tab:notations}, we summarize the notations used throughout this paper for clarity and ease of reference.

\begin{table}[htbp]
  \centering
  \caption{Summary of Notations}
  \label{tab:notations}
  \renewcommand{\arraystretch}{1.15}
  \begin{adjustbox}{width=0.92\columnwidth}
    \begin{tabular}{cl}
      \toprule
      \textbf{Notation} & \textbf{Description} \\
      \midrule
      $\mathcal{L}$ & Cross-entropy loss function \\
      $\bm{\theta}_\text{0}$ & Initialized global model parameters \\
      $\mathbf{W}$ & Weights of the FC/Conv layer \\
      $\mathbf{x} / \mathbf{X}$ & Input features of the FC/Conv layer \\
      $\mathbf{z} / \mathbf{Z}$ & Output values of the FC/Conv layer \\
      $\mathbf{a} / \mathbf{A}$ & Output activations of the FC/Conv layer \\
      $\overline{\nabla\bm{\theta}}_\text{0}$ & Gradient of the initialized global model \\
      $\nabla\mathbf{W}$ & Gradient with respect to weights \\
      $\nabla\mathbf{x} / \nabla\mathbf{X}$ & Gradient with respect to input features \\
      $\nabla\mathbf{z} / \nabla\mathbf{Z}$ & Gradient with respect to output values \\
      $\nabla\mathbf{a} / \nabla\mathbf{A}$ & Gradient with respect to output activations \\
      $\mathbf{p}$ & Softmax probabilities \\
      $\mathbf{y}$ & One-hot encoded labels \\
      $C$ & Number of classes \\
      $B$ & Batch size of training data \\
      $\mathcal{D}_\text{aux}$ & Auxiliary dataset for estimation \\
      \bottomrule
    \end{tabular}
  \end{adjustbox}
\end{table}

\subsection{Gradient Correlations in Typical Layers}

We begin by outlining the fundamental relationships and formulations of gradients in common neural network layers, namely the fully-connected (FC) layer, convolutional (Conv) layer, and ReLU activation layer. For simplicity, bias terms are omitted in the following discussion.

\subsubsection{FC Layer}

The forward propagation in an FC layer can be formulated as $\mathbf{z} = \mathbf{W} \mathbf{x}$, where $\mathbf{z} \in \mathbb{R}^{N}$ denotes the output vector, $\mathbf{W} \in \mathbb{R}^{N \times M}$ represents the weight matrix, and $\mathbf{x} \in \mathbb{R}^{M}$ denotes the input feature vector.

During backpropagation, the gradient of the loss function $\mathcal{L}$ with respect to the weight matrix $\mathbf{W}$ is given by:
\begin{equation}
  \label{eq:fc_grad_w}
  \nabla\mathbf{W} = \frac{\partial \mathcal{L}}{\partial \mathbf{W}} = \frac{\partial \mathcal{L}}{\partial \mathbf{z}} \frac{\partial \mathbf{z}}{\partial \mathbf{W}} = \nabla\mathbf{z} \mathbf{x}^{\top},
\end{equation}
where $\nabla\mathbf{z} \in \mathbb{R}^{N}$ denotes the gradient of the loss function with respect to the output $\mathbf{z}$ of the layer.

Similarly, the gradient of the loss function $\mathcal{L}$ relative to the input features $\mathbf{x}$ can be derived as:
\begin{equation}
  \label{eq:fc_grad_x}
  \nabla\mathbf{x} = \frac{\partial \mathcal{L}}{\partial \mathbf{x}} = \frac{\partial \mathcal{L}}{\partial \mathbf{z}} \frac{\partial \mathbf{z}}{\partial \mathbf{x}} = \mathbf{W}^{\top} \nabla\mathbf{z}.
\end{equation}

\subsubsection{Conv Layer}

In the Conv layer, the weights, also called kernels, are shared across different spatial locations of the input feature maps. A kernel is typically represented as a 4D tensor, denoted as $\mathbf{W} \in \mathbb{R}^{C_{\text{out}} \times C_{\text{in}} \times H_{\text{k}} \times W_{\text{k}}}$, where $C_{\text{in}}$ and $C_{\text{out}}$ are the numbers of input and output channels, respectively, and $H_{\text{k}}$ and $W_{\text{k}}$ represent the height and width of the kernel.
The forward pass can be expressed as $\mathbf{Z} = \mathbf{W} \ast \mathbf{X}$, where $\mathbf{Z} \in \mathbb{R}^{C_{\text{out}} \times H_{\text{out}} \times W_{\text{out}}}$ is the output feature map, $\mathbf{X} \in \mathbb{R}^{C_{\text{in}} \times H_{\text{in}} \times W_{\text{in}}}$ is the input feature map, and $\ast$ denotes the convolution operation. Unfolding the convolutional operation, a specific element $\mathrm{Z}_{k, i, j}$ in the output $\mathbf{Z}$ is computed as follows:
\begin{equation}
  \begin{aligned}
    \mathrm{Z}_{k, i, j}
     & = \sum_{c=1}^{C_{\text{in}}} \sum_{h=1}^{H_{\text{k}}} \sum_{w=1}^{W_{\text{k}}} \mathrm{W}_{k, c, h, w} \, \mathrm{X}_{c, i+h-1, j+w-1} \\
     & = \sum_{c, h, w} \mathrm{W}_{k, c, h, w} \, \mathrm{X}_{c, i+h-1, j+w-1},
  \end{aligned}
\end{equation}
where $k$ indexes the output channel, $i$ and $j$ index the spatial position of the output, $c$ indexes the input channels, and $h$ and $w$ index the spatial dimensions of the kernel. For brevity, we use the Einstein summation convention, such as $\sum_{c, h, w}$, to simplify the notation.

During the backward pass, the gradient with respect to the element $\mathrm{W}_{k, c, h, w}$ of the kernel $\mathbf{W}$ can be given by:
\begin{equation}
  \label{eq:conv_grad_w}
  \begin{aligned}
    \nabla\mathrm{W}_{k, c, h, w}
     & = \sum_{i=1}^{H_{\text{out}}} \sum_{j=1}^{W_{\text{out}}} \nabla\mathrm{Z}_{k, i, j} \, \mathrm{X}_{c, i+h-1, j+w-1} \\
     & = \sum_{i, j} \nabla\mathrm{Z}_{k, i, j} \, \mathrm{X}_{c, i+h-1, j+w-1}.
  \end{aligned}
\end{equation}

\subsubsection{ReLU Activation Layer}

The Rectified Linear Unit (ReLU) is the most widely used activation function in deep neural networks. The ReLU activation function is defined as $\mathrm{a} = \mathrm{ReLU}(\mathrm{z}) = \max(0, \mathrm{z})$, where $\mathrm{a}$ is the output, $\mathrm{z}$ is the input, and $\max(\cdot)$ denotes the maximum operation. The derivative of the output $\mathrm{a}$ with respect to the input $\mathrm{z}$ can be expressed as:
\begin{equation}
  \sigma'(\mathrm{z}) = \frac{\mathrm{d} \mathrm{a}}{\mathrm{d} \mathrm{z}} = 
  \left\{
  \begin{array}{@{}c@{\quad}l@{}}
    1, & \text{if } \mathrm{z} > 0,\\
    0, & \text{otherwise}.
  \end{array}
  \right.
\end{equation}

During backpropagation, the gradient with respect to an element $\mathrm{z}_k$ in the input vector $\mathbf{z}$ can be computed as:
\begin{equation}
  \label{eq:relu_grad_z}
  \nabla\mathrm{z}_k = \frac{\partial \mathcal{L}}{\partial \mathrm{a}_k} \frac{\partial \mathrm{a}_k}{\partial \mathrm{z}_k} =
  \left\{
    \begin{array}{@{}c@{\quad}l@{}}
      \nabla\mathrm{a}_k, & \text{if } \mathrm{z}_k > 0,\\
      0, & \text{otherwise}.
    \end{array}
  \right.
\end{equation}

\subsection{Relationship Between Logits and Labels}
\label{sec:logits_labels}

In a classification task with $C$ output classes, the cross-entropy loss function $\mathcal{L}$ is widely used to measure the discrepancy between the predicted Softmax probabilities $\mathbf{p} \in \mathbb{R}^{C}$ and the target one-hot labels $\mathbf{y} \in \mathbb{R}^{C}$. The cross-entropy loss is defined as:
\begin{equation}
  \label{eq:cross_entropy}
  \mathcal{L} = - \sum_{i=1}^{C} \mathrm{y}_i \log(\mathrm{p}_i) = -\log \frac{\exp(\mathrm{z}_c)}{\sum_{j=1}^{C} \exp(\mathrm{z}_j)},
\end{equation}
where $c$ represents the index of the target class, $\mathrm{z}_i$ is the $i$-th element of the output logits $\mathbf{z}$, and $\mathrm{p}_i = \frac{\exp(\mathrm{z}_i)}{\sum_{j=1}^{C} \exp(\mathrm{z}_j)}$ is the Softmax probability for class $i$.

During backward propagation, the gradient of the loss function $\mathcal{L}$ with respect to the logit $\mathrm{z}_i$ is derived as:
\begin{equation}
  \label{eq:cross_entropy_grad}
  \begin{aligned}
    \frac{\partial \mathcal{L}}{\partial \mathrm{z}_i}
    & = -\frac{\partial \log \exp(\mathrm{z}_c)}{\partial \mathrm{z}_i} + \frac{\partial \log \sum_{j=1}^{C} \exp(\mathrm{z}_j)}{\partial \mathrm{z}_i} \\
    & = -\delta_{ic} + \frac{\exp(\mathrm{z}_i)}{\sum_{j=1}^{C} \exp(\mathrm{z}_j)},
  \end{aligned}
\end{equation}
where $\delta_{ic}$ denotes the Kronecker delta function, which equals 1 when $i=c$ and 0 otherwise. Note that $y_i = \delta_{ic}$, so the Kronecker delta simply represents the one-hot label. Therefore, the gradients with respect to the logits vector $\mathbf{z}$ can be rewritten as:
\begin{equation}
  \label{eq:prob_label_grad}
  \nabla\mathbf{z} = \frac{\partial \mathcal{L}}{\partial \mathbf{z}} = \mathbf{p} - \mathbf{y}.
\end{equation}

During standard training, Eq.~\eqref{eq:prob_label_grad} reveals that the gradients $\nabla\mathbf{z}$ capture the difference between the predicted probabilities $\mathbf{p}$ and the ground-truth one-hot labels $\mathbf{y}$. By rearranging this relationship, the labels can be expressed as a function of the gradients and probabilities:
\begin{equation}
  \label{eq:label_from_grad}
  \mathbf{y} = \mathbf{p} - \nabla\mathbf{z}.
\end{equation}

Eq.~\eqref{eq:label_from_grad} shows that, if an adversary is able to observe or estimate both the gradients $\nabla\mathbf{z}$ and the probabilities $\mathbf{p}$, they can directly reconstruct the labels $\mathbf{y}$. This insight provides the theoretical foundation for label recovery attacks that leverage gradient information in FL.

\subsection{Threat Model}

We analyze a representative FL setting based on the FedSGD algorithm \cite{mcmahan2017communication}, which has been widely adopted in prior gradient inversion attacks \cite{zhu2019deep,yin2021see,geng2021towards,wainakh2022user}. In our scenario, the FL protocol employs encryption mechanisms to protect the gradients of the classification head's final layer, ensuring this component remains confidential during training.

We assume the FL server itself acts as the adversary, attempting to infer clients' label distributions from the available gradient information. The adversary operates in a white-box setting, with full knowledge of the model architecture, training hyperparameters (e.g., batch size), and the overall FL protocol. Since the server coordinates the training procedure, they can modify the global model prior to distribution. 


At the start of FL training, the server initializes and distributes the complete set of model parameters to all clients to accelerate convergence. During the first communication round, the server retains access to the entire parameter set. In subsequent rounds, however, encryption of the classification head restricts the server to receiving only gradients from the feature extractor. Consequently, label recovery attacks are feasible primarily during the initialization phase, when the server can access the full model parameters. Furthermore, all client updates are transmitted directly to the server in a point-to-point manner without secure aggregation, allowing a malicious server to inspect each client's gradients individually.

The adversary can exploit auxiliary resources, such as public datasets or statistical estimates of intermediate features, to enhance the effectiveness of attacks. Leveraging such information can refine reconstruction strategies and improve the accuracy of inferring clients' label distributions.

\section{Methodology}

In this section, we introduce our GDBR attack for label distribution leakage. We begin by outlining the theoretical foundations of GDBR, including the relationships between layer-wise gradients and the gradients with respect to output logits. Subsequently, we provide a detailed formulation for batch label restoration, utilizing estimated variables to enhance the attack's effectiveness.

\subsection{Overview}

To utilize Eq.~\eqref{eq:label_from_grad} for label recovery, the adversary must infer the gradients $\nabla\mathbf{z}$ with respect to the logits from the observable feature gradients, and further approximate variables (e.g., probabilities $\mathbf{p}$) based on model parameters and auxiliary data. In practice, the adversary has access to the gradients of the feature extractor, but GDBR requires only the gradient of the extractor's final layer. For representative model architectures, we consider two categories:
\begin{itemize}
  \item \textbf{MLP models}: Although the final FC layer is typically regarded as the classification head, in some cases the head may consist of multiple stacked FC layers. In MLPs, our analysis focuses on the stacked structure of multiple FC layers.
  \item \textbf{CNN models}: For CNNs, we examine AlexNet, LeNet, ResNet, and VGG, in which the last pooling layer is replaced by a depthwise convolution. These architectures consist of one DW convolution followed by one to three FC layers.
\end{itemize}

In the following, we analyze how to construct gradient bridges between layers in these settings, enabling the step-by-step inference of the logits gradient $\nabla\mathbf{z}$.

\subsection{Theoretical Foundations: Layer-Wise Gradient Correlations in FC, Conv, and ReLU Layers}

In this subsection, we derive the correlations between gradients in common neural network layers, including FC, Conv, and ReLU layers. These analytical results provide the theoretical foundation for the GDBR attack, enabling label recovery from partial gradient information in FL.

\begin{lemma}
  \label{lemma:fc_basic}
  In an FC layer, let $\nabla\mathbf{x}$, $\nabla\mathbf{W}$, and $\nabla\mathbf{z}$ denote the gradients with respect to the input $\mathbf{x}$, the weights $\mathbf{W}$, and the output $\mathbf{z}$, respectively. The following relationships then hold:
  \begin{align}
    \nabla\mathbf{x} \mathbf{x}^{\top} & = \mathbf{W}^{\top} \nabla\mathbf{W} \label{eq:fc_gx_xT}, \\
    \nabla\mathbf{z} \mathbf{z}^{\top} & = \nabla\mathbf{W} \mathbf{W}^{\top} \label{eq:fc_gz_zT}, \\
    \nabla\mathbf{z} \odot \mathbf{z} & = \mathrm{diag}(\nabla\mathbf{W} \mathbf{W}^{\top}), \label{eq:fc_gz_odot_z}
  \end{align}
  where $\odot$ denotes the element-wise product, and $\mathrm{diag}(\cdot)$ denotes the diagonal elements of a matrix.
\end{lemma}

\begin{proof}
  By multiplying $\mathbf{x}^{\top}$ on both sides of Eq.~\eqref{eq:fc_grad_x}, and substituting Eq.~\eqref{eq:fc_grad_w} into the result, we can derive:
  \begin{equation*}
    \nabla\mathbf{x} \mathbf{x}^{\top} = \mathbf{W}^{\top} \nabla\mathbf{z} \mathbf{x}^{\top} = \mathbf{W}^{\top} \nabla\mathbf{W}.
  \end{equation*}

  Similarly, if we multiply $\mathbf{W}^{\top}$ on both sides of Eq.~\eqref{eq:fc_grad_w}, we can obtain the following equation:
  \begin{equation*}
    \nabla\mathbf{W} \mathbf{W}^{\top} = \nabla\mathbf{z} \mathbf{x}^{\top} \mathbf{W}^{\top} = \nabla\mathbf{z} \left( \mathbf{W} \mathbf{x} \right)^{\top} = \nabla\mathbf{z} \mathbf{z}^{\top}.
  \end{equation*}

  Since $\nabla\mathbf{z} \mathbf{z}^{\top}$ is symmetric, i.e. $\left( \nabla\mathbf{z} \mathbf{z}^{\top} \right)^{\top} = \nabla\mathbf{z} \mathbf{z}^{\top}$, its diagonal entries correspond to the element-wise products of $\mathbf{z}$ with its gradient $\nabla\mathbf{z}$. Therefore, we have:
  \begin{equation*}
    \nabla\mathbf{z} \odot \mathbf{z} = \mathrm{diag}( \nabla\mathbf{z} \mathbf{z}^{\top} ) = \mathrm{diag}( \nabla\mathbf{W} \mathbf{W}^{\top} ),
  \end{equation*}
  where $\odot$ denotes the element-wise product.
\end{proof}

\begin{lemma}
  \label{lemma:fc_grad}
  In an FC layer, given the weight parameters $\mathbf{W}$ and the input feature gradient $\nabla\mathbf{x}$, the output gradient $\nabla\mathbf{z}$ can be expressed as follows:
  \begin{equation}
    \label{eq:fc_gz_new}
    \nabla\mathbf{z} = (\mathbf{W} \mathbf{W}^{\top})^{-1} \mathbf{W} \nabla\mathbf{x}.
  \end{equation}
\end{lemma}

\begin{proof}
  Left-multiplying the derived Eq.~\eqref{eq:fc_gx_xT} in Lemma~\ref{lemma:fc_basic} by $\mathbf{W}$ on both sides yields:
  \begin{equation*}
    \mathbf{W} \nabla\mathbf{x} \mathbf{x}^{\top} = \mathbf{W} \mathbf{W}^{\top} \nabla\mathbf{W}.
  \end{equation*}

  Since $\mathbf{W} \mathbf{W}^{\top}$ is actually invertible, we can further left-multiply both sides by $(\mathbf{W} \mathbf{W}^{\top})^{-1}$ to obtain:
  \begin{equation*}
    (\mathbf{W} \mathbf{W}^{\top})^{-1} \mathbf{W} \nabla\mathbf{x} \mathbf{x}^{\top} = \nabla\mathbf{W}.
  \end{equation*}

  Recall from Eq.~\eqref{eq:fc_grad_w} that $\nabla\mathbf{W} = \nabla\mathbf{z} \mathbf{x}^{\top}$. Substituting this expression into the previous equation results in:
  \begin{equation*}
    (\mathbf{W} \mathbf{W}^{\top})^{-1} \mathbf{W} \nabla\mathbf{x} \mathbf{x}^{\top} = \nabla\mathbf{z} \mathbf{x}^{\top}.
  \end{equation*}

  Finally, we isolate $\nabla\mathbf{z}$ by right-multiplying both sides by $\mathbf{x}(\mathbf{x}^{\top} \mathbf{x})^{-1}$. Since $\mathbf{x}^{\top} \mathbf{x} (\mathbf{x}^{\top} \mathbf{x})^{-1} = 1$, we derive:
  \begin{equation*}
    \begin{aligned}
      \nabla\mathbf{z}
      & = (\mathbf{W} \mathbf{W}^{\top})^{-1} \mathbf{W} \nabla\mathbf{x} \mathbf{x}^{\top} \mathbf{x} (\mathbf{x}^{\top} \mathbf{x})^{-1} \\
      & = (\mathbf{W} \mathbf{W}^{\top})^{-1} \mathbf{W} \nabla\mathbf{x}.
    \end{aligned}
  \end{equation*}
\end{proof}

\begin{myframed}
  \textbf{Remark:} Lemma~\ref{lemma:fc_grad} establishes that the gradient $\nabla\mathbf{z}$ in an FC layer can be explicitly recovered from the input gradient $\nabla\mathbf{x}$, using the weight parameters $\mathbf{W}$.
\end{myframed}

\begin{lemma}
  \label{lemma:conv_grad}
  In a Conv layer, let $\nabla\mathbf{W}_k$ and $\nabla\mathbf{Z}_k$ represent the gradients for the $k$-th kernel $\mathbf{W}_k$ and the output $\mathbf{Z}_k$, respectively. Then the following relationship holds:
  \begin{equation}
    \langle \nabla\mathbf{W}_k, \mathbf{W}_k \rangle_\text{F} =
    \left\{
      \renewcommand{\arraystretch}{1.25}
      \begin{array}{@{}c@{\quad}l@{}}
        \nabla\mathbf{Z}_k \odot \mathbf{Z}_k, & \text{ if } \mathbf{Z}_k \in \mathbb{R}, \\
        \langle \nabla\mathbf{Z}_k, \mathbf{Z}_k \rangle_\text{F}, & \text{ otherwise},
      \end{array}
    \right.
  \end{equation}
  where $\odot$ denotes the element-wise product, and $\langle \cdot, \cdot \rangle_\text{F}$ denotes the Frobenius inner product that calculates the sum of the element-wise product of two tensors or matrices.
\end{lemma}

\begin{proof}
  If we multiply $\mathrm{W}_{k, c, h, w}$ on both sides of Eq.~\eqref{eq:conv_grad_w}, and then sum over $c$, $h$, and $w$, we can derive:
  \begin{equation*}
    \begin{aligned}
      \sum_{\substack{c, h, w}} & \nabla\mathrm{W}_{k, c, h, w} \, \mathrm{W}_{k, c, h, w} \\
      & = \sum_{\substack{c, h, w}} \sum_{\substack{i, j}} \nabla\mathrm{Z}_{k, i, j} \, \mathrm{X}_{c, h+i-1, w+j-1} \, \mathrm{W}_{k, c, h, w} \\
      & = \sum_{\substack{i, j}} \nabla\mathrm{Z}_{k, i, j} \, \sum_{\substack{c, h, w}} \mathrm{W}_{k, c, h, w} \, \mathrm{X}_{c, i+h-1, j+w-1} \\
      & = \sum_{\substack{i, j}} \nabla\mathrm{Z}_{k, i, j} \, \mathrm{Z}_{k, i, j}.
    \end{aligned}
  \end{equation*}

  We can thus simplify this equation as $\langle \nabla\mathbf{W}_k, \mathbf{W}_k \rangle_\text{F} = \langle \nabla\mathbf{Z}_k, \mathbf{Z}_k \rangle_\text{F}$ using the Frobenius inner product. When the shape of the output $\mathbf{Z}_k$ is $1 \times 1$, that is, $\mathbf{Z}_k$ is a scalar, we can represent the result using the element-wise product $\odot$, so that $\langle \nabla\mathbf{W}_k, \mathbf{W}_k \rangle_\text{F} = \nabla\mathbf{Z}_k \odot \mathbf{Z}_k$.
\end{proof}

\begin{myframed}
  \textbf{Remark:} Since the GAP layer is replaced by a DW convolution with a $1 \times 1$ output, the operation simplifies to an element-wise product. Thus, we employ the form $\nabla\mathbf{Z}_k \odot \mathbf{Z}_k$ to construct the gradient bridge.
\end{myframed}

\begin{lemma}
  \label{lemma:relu_grad}
  In a ReLU layer, let $\nabla\mathbf{z}$ and $\nabla\mathbf{a}$ represent the gradients with respect to the input $\mathbf{z}$ and output $\mathbf{a}$, respectively. Then the following relationship holds:
  \begin{equation}
    \label{eq:relu_gz_odot_z}
    \nabla\mathbf{z} \odot \mathbf{z} = \nabla\mathbf{a} \odot \mathbf{a},
  \end{equation}
  where $\odot$ denotes the element-wise product. The identity $\nabla\mathbf{Z} \odot \mathbf{Z} = \nabla\mathbf{A} \odot \mathbf{A}$ holds regardless of whether $\mathbf{Z}$ and $\mathbf{A}$ are matrices or tensors.
\end{lemma}

\begin{proof}
  By multiplying $\mathrm{z}_k$ on both sides of Eq.~\eqref{eq:relu_grad_z}, for each element $k$, we can deduce the following relationship:
  \begin{equation*}
    \nabla\mathrm{z}_k \, \mathrm{z}_k = \nabla\mathrm{a}_k \, \max(0, \mathrm{z}_k) = \nabla\mathrm{a}_k \, \mathrm{a}_k.
  \end{equation*}

  The relationship extends to the entire vector $\mathbf{z}$, yielding $\nabla\mathbf{z} \odot \mathbf{z} = \nabla\mathbf{a} \odot \mathbf{a}$ for each element in this ReLU layer. The same remains true for variables in vector, matrix, or tensor form.
\end{proof}

\begin{theorem}
  \label{theorem:fc_relu_stack}
  In a single stack consisting of an FC layer followed by a ReLU activation (FC-ReLU), the gradient $\nabla\mathbf{a}$ can be inferred from the following relationships:
  \begin{equation}
    \label{eq:fc_relu_ga}
    \nabla\mathbf{a} =
    \left\{
      \renewcommand{\arraystretch}{1.25}
      \begin{array}{@{}c@{\quad}l@{}}
        ( \mathbf{W} \mathbf{W}^{\top} )^{-1} \mathbf{W} \nabla\mathbf{x}, & \text{if } \nabla\mathbf{x} \text{ is given}, \\
        \mathrm{diag}( \nabla\mathbf{W} \mathbf{W}^{\top} ) \oslash \mathbf{a}, & \text{if } \nabla\mathbf{W} \text{ is given},
      \end{array}
    \right.
  \end{equation}
  where $\oslash$ denotes the element-wise division, and all elements in $\mathbf{a}$ are supposed to be nonzero in the second case.
\end{theorem}

\begin{proof}
  Combining Lemma~\ref{lemma:fc_grad} and Lemma~\ref{lemma:relu_grad}, we can establish the connection between the activated output gradient $\nabla\mathbf{a}$ and the input gradient $\nabla\mathbf{x}$ in the FC-ReLU stack as:
  \begin{equation*}
    \nabla\mathbf{a} = \nabla\mathbf{z} \odot \mathbf{z} \oslash \mathbf{a} = ( \mathbf{W} \mathbf{W}^{\top} )^{-1} \mathbf{W} \nabla\mathbf{x} \odot \mathbf{z} \oslash \mathbf{a}.
  \end{equation*}

  Since all the elements in $\mathbf{a}$ are supposed to be nonzero, it is obvious that $\mathbf{a} = \mathbf{z}$ holds from the definition of ReLU activation function. In this case, $\mathbf{z} \oslash \mathbf{a}$ can be simplified as $\mathbf{1}$, where all the elements in $\mathbf{1}$ are equal to 1. Therefore, we can simplify the above equation, that is:
  \begin{equation*}
    \nabla\mathbf{a} = ( \mathbf{W} \mathbf{W}^{\top} )^{-1} \mathbf{W} \nabla\mathbf{x}.
  \end{equation*}

  Combining Eq.~\eqref{eq:fc_gz_odot_z} and Eq.~\eqref{eq:relu_gz_odot_z}, there is another way to infer the gradient $\nabla\mathbf{a}$ when the gradient $\nabla\mathbf{W}$ and the output $\mathbf{a}$ are given. The derivation of $\nabla\mathbf{a}$ is as follows:
  \begin{equation*}
    \nabla\mathbf{a} = \mathrm{diag}( \nabla\mathbf{W} \mathbf{W}^{\top} ) \oslash \mathbf{a},
  \end{equation*}
  where all elements in $\mathbf{a}$ are supposed to be nonzero.
\end{proof}

\begin{myframed}
  \textbf{Remark:} Theorem~\ref{theorem:fc_relu_stack} demonstrates that the gradient $\nabla\mathbf{a}$ is recoverable via two distinct formulations. The presence of $\nabla\mathbf{x}$ allows derivation via the first case, while the knowledge of $\nabla\mathbf{W}$ enables the second.
\end{myframed}

\begin{theorem}
  \label{theorem:conv_relu_stack}
  In a single stack consisting of a Conv layer followed by a ReLU activation (Conv-ReLU), the gradient $\nabla\mathbf{A}\in\mathbb{R}^{N}$ can be inferred from the following equations:
  \begin{equation}
    \left\{
      \renewcommand{\arraystretch}{1.25}
      \begin{array}{@{}c@{\hspace{0.2em}}c@{\hspace{0.2em}}l}
        \nabla\mathbf{A}_k & = & \langle \nabla\mathbf{W}_k, \mathbf{W}_k \rangle_\text{F} \oslash \mathbf{A}_k, \\
        \nabla\mathbf{A} & = & [ \nabla\mathbf{A}_1, \nabla\mathbf{A}_2, \cdots, \nabla\mathbf{A}_K ]^{\top},
      \end{array}
    \right.
  \end{equation}
  where $\oslash$ denotes the element-wise division, and the activation value $\mathbf{A}_k$ is supposed to be nonzero.
\end{theorem}

\begin{proof}
  Based on Lemma~\ref{lemma:conv_grad} and Lemma~\ref{lemma:relu_grad}, the relationships $\nabla\mathbf{Z} \odot \mathbf{Z} = \nabla\mathbf{A} \odot \mathbf{A}$ and $\nabla\mathbf{Z} \odot \mathbf{Z} = \langle \nabla\mathbf{W}, \mathbf{W} \rangle_\text{F}$ hold in the Conv-ReLU stack, where $\mathbf{Z}$ and $\mathbf{A}$ denote the outputs of the Conv layer and the ReLU layer, respectively.

  In the context of the GDBR attack, the adversary has replaced the GAP layer with a depthwise convolution. Consequently, the spatial output of this Conv layer is $1 \times 1$, meaning $\mathbf{A}_k \in \mathbb{R}$ is a scalar for the $k$-th kernel, and the full activation gradient is a vector $\nabla\mathbf{A} \in \mathbb{R}^{N}$, where $N$ is the number of kernels.

  Leveraging this dimensionality, we isolate the gradient for the $k$-th kernel. Assuming $\mathbf{A}_k$ is nonzero, we obtain:
  \begin{equation*}
    \nabla\mathbf{A}_k = \langle \nabla\mathbf{W}_k, \mathbf{W}_k \rangle_\text{F} \oslash \mathbf{A}_k.
  \end{equation*}

  Finally, the total gradient $\nabla\mathbf{A}$ is derived by concatenating the individual gradients $\nabla\mathbf{A}_k$ for all kernels.
\end{proof}

\begin{myframed}
  \textbf{Remark:} For a Conv-ReLU stack, Theorem~\ref{theorem:conv_relu_stack} shows that the $k$-th activated feature gradient $\nabla\mathbf{A}_k$ can be inferred from the observable weight gradient $\nabla\mathbf{W}_k$ and the approximately estimated feature $\mathbf{A}_k$.
\end{myframed}

\subsection{The Gradient Bridge in Stacked Layers}

The classification heads of deep neural networks are typically composed of sequential layers involving linear transformations followed by nonlinear activations. We formulate a general architecture consisting of $L$ sequentially stacked layers to unify the analysis of standard MLPs and the specific CNN post-processing stages (where the GAP layer is replaced by a depthwise convolution).

In this framework, the initial layer serves as the interface from features to the classifier, followed by a series of intermediate representations, ending with the final output logits. We denote the layer index by the superscript $[l]$, where $l \in \{1, 2, \ldots, L\}$. The specific configuration of these layers is defined as follows:
\begin{itemize}
  \item \textbf{Input layer ($l = 1$)}: A Conv-ReLU stack (for CNNs) or an FC-ReLU stack (for MLPs);
  \item \textbf{Hidden layers ($2 \leq l \leq L-1$)}: A sequence of standard FC-ReLU stacks;
  \item \textbf{Output layer ($l = L$)}: A final FC layer without ReLU activation.
\end{itemize}

\begin{theorem}
  \label{theorem:recursive_stack_grad}
  In an $L$-layer network comprising an initial FC/Conv-ReLU layer, $(L-2)$ intermediate FC-ReLU layers, and a final FC layer, the gradients of the features can be recovered recursively:

  \begin{enumerate}
    \item \textbf{Layer $l = 1$:} The activation gradient for the $k$-th feature is:
    \begin{equation}
      \label{eq:first_stack_grad}
      \nabla\mathbf{a}^{[1]}_k =
      \left\{
      \begin{array}{@{}ll}
        \mathrm{diag}( \nabla\mathbf{W}^{[1]}_k \mathbf{W}^{[1]\top}_k ) \oslash \mathbf{a}^{[1]}_k, & (\text{FC}) \\
        \langle \nabla\mathbf{W}^{[1]}_k, \mathbf{W}^{[1]}_k \rangle_\mathrm{F} \oslash \mathbf{a}^{[1]}_k, & (\text{Conv})
      \end{array}
      \right.
    \end{equation}
    
    We set $\nabla\mathbf{x}^{[2]} = [ \nabla\mathbf{a}_1^{[1]}, \ldots, \nabla\mathbf{a}_K^{[1]} ]^{\top}$.

    \item \textbf{Layers $2 \leq l \leq L-1$:} The input gradient for the subsequent layer is:
    \begin{equation}
      \label{eq:recursive_stack_grad}
      \nabla\mathbf{x}^{[l+1]} = ( \mathbf{W}^{[l]} \mathbf{W}^{[l]\top} )^{-1} \mathbf{W}^{[l]} \nabla\mathbf{x}^{[l]}.
    \end{equation}

    \item \textbf{Layer $l = L$:} The final logit gradient is:
    \begin{equation}
      \nabla\mathbf{z}^{[L]} = (\mathbf{W}^{[L]} \mathbf{W}^{[L]\top})^{-1} \mathbf{W}^{[L]} \nabla\mathbf{x}^{[L]}.
    \end{equation}
  \end{enumerate}
\end{theorem}

\begin{proof}
  By invoking Theorem~\ref{theorem:fc_relu_stack} and Theorem~\ref{theorem:conv_relu_stack}, we determine the activation gradient $\nabla\mathbf{a}^{[1]}_k$ for the initial layer. Given that the weight gradient $\nabla\mathbf{W}^{[1]}_k$ is observable from the shared gradient updates, we have:
  \begin{equation*}
    \nabla\mathbf{a}^{[1]}_k =
    \begin{cases}
      \mathrm{diag}( \nabla\mathbf{W}^{[1]}_k \mathbf{W}^{[1]\top}_k ) \oslash \mathbf{a}^{[1]}_k & \text{if FC}, \\[1ex]
      \langle \nabla\mathbf{W}^{[1]}_k, \mathbf{W}^{[1]}_k \rangle_\mathrm{F} \oslash \mathbf{a}^{[1]}_k & \text{if Conv}.
    \end{cases}
  \end{equation*}

  For the subsequent intermediate layers $l \in \{2, \ldots, L-1\}$, we utilize the recursive relationship established in Theorem~\ref{theorem:fc_relu_stack}. Recognizing that the input to layer $l+1$ is the activated output of layer $l$ (i.e., $\nabla\mathbf{x}^{[l+1]} = \nabla\mathbf{a}^{[l]}$), the inference of the feature gradients proceeds as:
  \begin{equation*}
    \begin{aligned}
      \nabla\mathbf{x}^{[2]} &= \big[ \nabla\mathbf{a}_1^{[1]}, \nabla\mathbf{a}_2^{[1]}, \ldots, \nabla\mathbf{a}_K^{[1]} \big]^{\top}, \\
      \nabla\mathbf{x}^{[l+1]} &= ( \mathbf{W}^{[l]} \mathbf{W}^{[l]\top} )^{-1} \mathbf{W}^{[l]} \nabla\mathbf{x}^{[l]}.
    \end{aligned}
  \end{equation*}

  Finally, applying Lemma~\ref{lemma:fc_basic} to the output layer ($l=L$), the gradient with respect to the logits $\nabla\mathbf{z}^{[L]}$ is derived as:
  \begin{equation*}
    \nabla\mathbf{z}^{[L]} = (\mathbf{W}^{[L]} \mathbf{W}^{[L]\top})^{-1} \mathbf{W}^{[L]} \nabla\mathbf{x}^{[L]}.
  \end{equation*}

  Together, these derivations establish a recursive framework for recovering feature gradients, propagating from the initial layer to the final output logits.
\end{proof}

We designate the derived recursive relationship as the \emph{gradient bridge}, establishing a direct pathway between the gradient of the initial FC/Conv layer and the output logits. Crucially, this formulation implies that recovering the target gradient $\nabla\mathbf{z}^{[L]}$ requires only the weight gradient of the first layer, $\nabla\mathbf{W}^{[1]}$, and the global model parameters $\bm{\theta}_\text{0}$. In the FL context, since $\bm{\theta}_\text{0}$ is broadcast to all participants and the activation $\mathbf{a}^{[1]}$ can be approximated using auxiliary data, the gradient bridge effectively circumvents FL's gradient masking. This mechanism serves as the theoretical foundation for our GDBR attack.

\begin{myframed}
  \textbf{Remark:} Theorem~\ref{theorem:recursive_stack_grad} formally establishes the \emph{gradient bridge}, showing that the logits gradient $\nabla\mathbf{z}^{[L]}$ is strictly recoverable given the initial layer's weight gradient $\nabla\mathbf{W}^{[1]}$ and the global model parameters.
\end{myframed}

\subsection{Label Recovery from Averaged Gradients}

While Theorem~\ref{theorem:recursive_stack_grad} establishes the \emph{gradient bridge} for individual samples, practical FL updates typically consist of gradients averaged over a batch. In this subsection, we generalize the GDBR framework to accommodate these aggregated updates. Specifically, we demonstrate how classwise labels can be recovered from the batch-averaged gradient by leveraging the estimated features $\widetilde{\mathbf{a}}^{[1]}$ and probability vectors $\widetilde{\mathbf{p}}$.

\begin{assumption}[Batch Feature Homogeneity]
  \label{assumption:feat_approx}
  We assume that the output features $\mathbf{a}^{[1](n)}$ of the initial layer exhibit low variance across samples within a batch of size $B$. Consequently, the individual features can be approximated by a single representative vector $\widetilde{\mathbf{a}}^{[1]} \in \mathbb{R}^K$, such that:
  \begin{equation}
    \mathbf{a}^{[1](n)} \approx \widetilde{\mathbf{a}}^{[1]}, \quad \forall n \in \{1, 2, \ldots, B\}.
  \end{equation}
\end{assumption}

\begin{assumption}[Batch Prediction Homogeneity]
  \label{assumption:prob_approx}
  We assume that the Softmax probability vectors $\mathbf{p}^{(n)}$ exhibit low variance across samples within a batch of size $B$. Consequently, the batch predictions can be approximated by a single representative probability vector $\widetilde{\mathbf{p}} \in \mathbb{R}^C$, such that:
  \begin{equation}
    \mathbf{p}^{(n)} \approx \widetilde{\mathbf{p}}, \quad \forall n \in \{1, 2, \ldots, B\}.
  \end{equation}
\end{assumption}

\begin{theorem}
  \label{theorem:batch_avg_grad}
  Consider the network architecture defined in Theorem~\ref{theorem:recursive_stack_grad}. Let the gradients be averaged over a batch of size $B$. Under Assumption~\ref{assumption:feat_approx}, the batch-averaged gradients can be recursively reconstructed as follows:

  \begin{enumerate}
    \item \textbf{Layer $l=1$:} The gradient is approximated using the feature estimate $\widetilde{\mathbf{a}}^{[1]}$:
    \begin{equation}
      \label{eq:batch_first_stack_grad}
      \overline{\nabla\mathbf{a}}^{[1]}_k \approx
      \left\{
      \begin{array}{@{}ll}
        \mathrm{diag}( \overline{\nabla\mathbf{W}}^{[1]}_k \mathbf{W}^{[1]\top}_k ) \oslash \widetilde{\mathbf{a}}^{[1]}_k, & (\text{FC}) \\
        \langle \overline{\nabla\mathbf{W}}^{[1]}_k, \mathbf{W}^{[1]}_k \rangle_\mathrm{F} \oslash \widetilde{\mathbf{a}}^{[1]}_k, & (\text{Conv})
      \end{array}
      \right.
    \end{equation}
    We set $\overline{\nabla\mathbf{x}}^{[2]} = [ \overline{\nabla\mathbf{a}}_1^{[1]}, \ldots, \overline{\nabla\mathbf{a}}_K^{[1]} ]^{\top}$.

    \item \textbf{Layers $2 \leq l \leq L-1$:} The input gradient for the subsequent layer is:
    \begin{equation}
      \label{eq:batch_recursive_stack_grad}
      \overline{\nabla\mathbf{x}}^{[l+1]} = ( \mathbf{W}^{[l]} \mathbf{W}^{[l]\top} )^{-1} \mathbf{W}^{[l]} \overline{\nabla\mathbf{x}}^{[l]}.
    \end{equation}

    \item \textbf{Layer $l = L$:} The final logit gradient is:
    \begin{equation}
      \label{eq:gdbr_grad_logits}
      \overline{\nabla\mathbf{z}}^{[L]} = ( \mathbf{W}^{[L]} \mathbf{W}^{[L]\top} )^{-1} \mathbf{W}^{[L]} \overline{\nabla\mathbf{x}}^{[L]}.
    \end{equation}
  \end{enumerate}
\end{theorem}

\begin{proof}
  We consider a batch training setting where gradients are averaged over $B$ samples. Invoking Assumption~\ref{assumption:feat_approx}, we approximate the feature $\mathbf{a}^{[1](n)}$ for every sample $n$ with the representative vector $\widetilde{\mathbf{a}}^{[1]}$.
  Combining Eq.~\eqref{eq:first_stack_grad} with this assumption, we can derive the averaged gradient $\overline{\nabla\mathbf{a}}^{[1]}_k$. Since $\widetilde{\mathbf{a}}^{[1]}$ is treated as a constant vector across the batch, it factors out of the summation:
  \begin{equation*}
    \begin{aligned}
      \overline{\nabla\mathbf{a}}^{[1]}_k
       & = \frac{1}{B} \sum_{n=1}^{B} \nabla\mathbf{a}^{[1](n)}_k
       \approx \frac{1}{B} \sum_{n=1}^{B} \left( \dots \right) \oslash \widetilde{\mathbf{a}}^{[1]}_k  \\
       & = \begin{cases}
          \mathrm{diag}( \overline{\nabla\mathbf{W}}^{[1]}_k \mathbf{W}^{[1]\top}_k ) \oslash \widetilde{\mathbf{a}}^{[1]}_k & \text{if FC}, \\[1ex]
          \langle \overline{\nabla\mathbf{W}}^{[1]}_k, \mathbf{W}^{[1]}_k \rangle_\mathrm{F} \oslash \widetilde{\mathbf{a}}^{[1]}_k & \text{if Conv},
        \end{cases}
    \end{aligned}
  \end{equation*}
  where $\overline{\nabla\mathbf{W}}^{[1]}_k = \frac{1}{B} \sum_{n=1}^{B} \nabla\mathbf{W}^{[1](n)}_k$ represents the batch-averaged weight gradient. Consequently, the input gradient to the second layer is:
  \begin{equation*}
    \overline{\nabla\mathbf{x}}^{[2]} = \overline{\nabla\mathbf{a}}^{[1]} = \big[ \overline{\nabla\mathbf{a}}_1^{[1]}, \ldots, \overline{\nabla\mathbf{a}}_K^{[1]} \big]^{\top}.
  \end{equation*}

  For the intermediate layers ($2 \leq l \leq L-1$), the gradient computation involves only linear operations (matrix multiplication and inversion). Exploiting the linearity of the expectation operator, the relationship holds exactly for batch-averaged quantities. Starting from Eq.~\eqref{eq:recursive_stack_grad}:
  \begin{equation*}
    \begin{aligned}
      \overline{\nabla\mathbf{x}}^{[l+1]} = \overline{\nabla\mathbf{a}}^{[l]}
      & = \frac{1}{B} \sum_{n=1}^{B} (\mathbf{W}^{[l]} \mathbf{W}^{[l]\top})^{-1} \mathbf{W}^{[l]} \nabla\mathbf{x}^{[l](n)} \\
      & = (\mathbf{W}^{[l]} \mathbf{W}^{[l]\top})^{-1} \mathbf{W}^{[l]} \frac{1}{B} \sum_{n=1}^{B} \nabla\mathbf{x}^{[l](n)} \\
      & = (\mathbf{W}^{[l]} \mathbf{W}^{[l]\top})^{-1} \mathbf{W}^{[l]} \overline{\nabla\mathbf{x}}^{[l]}.
    \end{aligned}
  \end{equation*}

  Finally, applying the same linearity logic to the output layer $L$, the averaged logit gradient is:
  \begin{equation*}
    \overline{\nabla\mathbf{z}}^{[L]} = ( \mathbf{W}^{[L]} \mathbf{W}^{[L]\top} )^{-1} \mathbf{W}^{[L]} \overline{\nabla\mathbf{x}}^{[L]}.
  \end{equation*}
  Thus, the batch-averaged gradients of these ``head layers'' are recursively computable as stated.
\end{proof}

Recalling Eq.~\eqref{eq:label_from_grad} from Section~\ref{sec:logits_labels}, we express the batch-averaged gradient $\overline{\nabla\mathbf{z}}^{[L]}$ as the mean of the differences between predictions and ground-truth labels:
\begin{equation}
  \label{eq:batch_avg_grad_logit}
  \overline{\nabla\mathbf{z}}^{[L]} = \frac{1}{B} \sum_{n=1}^{B} \nabla\mathbf{z}^{[L](n)} = \frac{1}{B} \sum_{n=1}^{B} \left( \mathbf{p}^{(n)} - \mathbf{y}^{(n)} \right).
\end{equation}

To recover the labels, we define the classwise count vector $\bm{\lambda} \in \mathbb{Z}^C$ as the sum of the one-hot label vectors across the batch, i.e., $\bm{\lambda} = \sum_{n=1}^{B} \mathbf{y}^{(n)}$. Applying Assumption~\ref{assumption:prob_approx}, we approximate the summation of predictions as $\sum_{n=1}^{B} \mathbf{p}^{(n)} \approx B \widetilde{\mathbf{p}}$. Substituting these terms into Eq.~\eqref{eq:batch_avg_grad_logit}, we can express the batch-averaged gradient as:
\begin{equation}
  \overline{\nabla\mathbf{z}}^{[L]} \approx \widetilde{\mathbf{p}} - \frac{1}{B}\bm{\lambda}.
\end{equation}

Rearranging this equation allows us to solve for the label distribution $\bm{\lambda}$:
\begin{equation}
  \label{eq:label_recovery}
  \bm{\lambda} = B \left( \widetilde{\mathbf{p}} - \overline{\nabla\mathbf{z}}^{[L]} \right).
\end{equation}

Thus, given the reconstructed gradient $\overline{\nabla\mathbf{z}}^{[L]}$ (derived in Eq.~\eqref{eq:gdbr_grad_logits}) and the estimated probability vector $\widetilde{\mathbf{p}}$, the classwise sample counts $\bm{\lambda}$ can be explicitly recovered. It is important to note that due to the approximation of $\widetilde{\mathbf{p}}$, the calculated elements of $\bm{\lambda}$ may be continuous values. However, since $\bm{\lambda}$ physically represents the number of samples per class, these values can be rectified simply by rounding them to the nearest non-negative integers subject to the constraint $\sum \lambda_c = B$. This step effectively reveals the exact label distribution of the private batch data.

\begin{myframed}
  \textbf{Remark:} Theorem~\ref{theorem:batch_avg_grad} and Eq.~\eqref{eq:label_recovery} together extends the \emph{gradient bridge} to batch-averaged gradients, enabling the recovery of classwise label distributions from the shared feature gradients in FL.
\end{myframed}

\subsection{Algorithm of GDBR Label Recovery}

Building upon the theoretical findings established in the preceding subsections, we synthesize the GDBR attack procedure in Algorithm~\ref{alg:gdbr_attack}. This method establishes the \emph{gradient bridge} to reconstruct the private label distribution $\bm{\lambda}$. It accepts the initialized global model $\bm{\theta}_\text{0}$, the shared batch-averaged gradient $\overline{\nabla\bm{\theta}}_\text{0}$, an auxiliary dataset $\mathcal{D}_\text{aux}$, and the batch size $B$ as inputs. The algorithm consists of three main stages: (1) estimating the variables needed to instantiate the gradient bridge, (2) propagating the accessible gradient to the missing logit gradient, and (3) converting the resulting continuous label estimate into a valid classwise label distribution.

First, GDBR extracts the accessible weight gradient $\overline{\nabla\mathbf{W}}^{[1]}$ from the shared update. Depending on where partial encryption starts, this gradient corresponds to the first unencrypted FC/Conv layer that can be connected to the encrypted output layer through the bridge. Since the activation $\mathbf{a}^{[1]}$ and the Softmax probabilities $\mathbf{p}$ of the private batch are not directly observable, GDBR estimates their batch-level representatives $\widetilde{\mathbf{a}}^{[1]}$ and $\widetilde{\mathbf{p}}$ using the auxiliary dataset $\mathcal{D}_\text{aux}$. These estimates instantiate Assumptions~\ref{assumption:feat_approx} and~\ref{assumption:prob_approx}, respectively.

Second, GDBR reconstructs the missing gradient with respect to the logits. It first computes the activation gradient $\overline{\nabla\mathbf{a}}^{[1]}$ from the observed weight gradient and the estimated feature vector using Eq.~\eqref{eq:batch_first_stack_grad}. The algorithm then treats this value as the input gradient to the next layer and recursively applies Eq.~\eqref{eq:batch_recursive_stack_grad} through the remaining head layers. This recursive propagation is the operational form of the gradient bridge: it maps the unencrypted feature-side gradient to the encrypted or unavailable classification-head signal. At the final layer, Eq.~\eqref{eq:gdbr_grad_logits} yields the reconstructed batch-averaged logit gradient $\overline{\nabla\mathbf{z}}^{[L]}$.

Third, GDBR solves Eq.~\eqref{eq:label_recovery} to obtain a continuous estimate of the classwise count vector. Ideally, $\bm{\lambda}$ should be a histogram of labels in the batch: each entry must be a non-negative integer and the total count must equal the batch size $B$. However, because $\widetilde{\mathbf{a}}^{[1]}$, $\widetilde{\mathbf{p}}$, and the reconstructed logit gradient are approximations, the direct solution of Eq.~\eqref{eq:label_recovery} can contain fractional values or small negative numerical artifacts. We therefore project the continuous estimate onto the feasible integer simplex:
\begin{equation}
  \label{eq:integer_projection}
  \begin{aligned}
    \hat{\bm{\lambda}} = 
    \arg\min_{\bm{\lambda}'}
    \quad & \left\| \bm{\lambda} - \bm{\lambda}' \right\|_2 \\
    \mathrm{s.t.}\quad
    & \sum_{c=1}^{C} \lambda'_c = B, \\
    & \lambda'_c \in \mathbb{Z}_{\geq 0},\quad \forall c \in \{1,\ldots,C\}.
  \end{aligned}
\end{equation}

This projection selects the closest valid label-count vector to the continuous estimate. The equality constraint ensures that the recovered distribution accounts for exactly all $B$ samples, while the integrality and non-negativity constraints ensure that each component can be interpreted as the number of samples from class $c$. In practice, this step behaves like constrained rounding: fractional counts are rounded while correcting the total count so that no sample is lost or duplicated. The output $\hat{\bm{\lambda}}$ is therefore a valid recovered label histogram for the private batch.

\begin{algorithm}[t]
  \caption{GDBR Label Recovery Attack}
  \label{alg:gdbr_attack}
  \begin{algorithmic}[1]
    \Require Initialized global model $\bm{\theta}_\text{0}$, batch-averaged gradient $\overline{\nabla\bm{\theta}}_\text{0}$, auxiliary dataset $\mathcal{D}_\text{aux}$, batch size $B$
    \Ensure Recovered classwise label distribution $\hat{\bm{\lambda}}$
    \State Extract weight gradient $\overline{\nabla\mathbf{W}}^{[1]}$ from $\overline{\nabla\bm{\theta}}_\text{0}$
    \State Estimate feature statistics $\widetilde{\mathbf{a}}^{[1]}$ using $\mathcal{D}_\text{aux}$
    \State Estimate probability statistics $\widetilde{\mathbf{p}}$ using $\mathcal{D}_\text{aux}$
    \State Compute activation gradient $\overline{\nabla\mathbf{a}}^{[1]}$ via Eq.~\eqref{eq:batch_first_stack_grad}
    \State Initialize recursive input: $\overline{\nabla\mathbf{x}}^{[2]} \leftarrow \overline{\nabla\mathbf{a}}^{[1]}$
    \For{$l = 2$ to $L-1$}
      \State Propagate $\overline{\nabla\mathbf{x}}^{[l+1]}$ using recursive bridge Eq.~\eqref{eq:batch_recursive_stack_grad}
    \EndFor
    \State Recover logits gradient $\overline{\nabla\mathbf{z}}^{[L]}$ using Eq.~\eqref{eq:gdbr_grad_logits}
    \State Solve for continuous label-count estimate $\bm{\lambda}$ using Eq.~\eqref{eq:label_recovery}
    \State $\hat{\bm{\lambda}} \leftarrow$ integer projection of $\bm{\lambda}$ using Eq.~\eqref{eq:integer_projection}
    \State \Return $\hat{\bm{\lambda}}$
  \end{algorithmic}
\end{algorithm}

\section{Experiments}

In this section, we empirically validate the effectiveness and versatility of the GDBR attack. We structure our evaluation into six key components: (1) \emph{empirical verification} of the underlying assumptions regarding representation homogeneity; (2) \emph{comparative benchmarking} against existing label restoration attacks; (3) \emph{sensitivity analysis} concerning batch size, layer depth, class distribution, and auxiliary information; (4) \emph{architectural generalization} to transformer-based vision models; (5) \emph{evaluation under local multi-step updates}; and (6) \emph{robustness assessment} under active gradient defense strategies.

\subsection{Experimental Setup}

\subsubsection{Datasets and Models}

We evaluate the GDBR attack across five image classification benchmarks, ranging from gray-scale digits to large-scale images. Here are the details of the evaluation datasets:
\begin{itemize}
  \item \textbf{MNIST} \cite{lecun1998gradient} contains 60k training and 10k testing gray-scale digits of size 28$\times$28 with 10 classes.
  \item \textbf{SVHN} \cite{netzer2011reading} contains 73k training and 26k testing color digit images of size 32$\times$32 with 10 classes.
  \item \textbf{CIFAR-10/100} \cite{krizhevsky2009learning} contains 50k training and 10k testing images of size 32$\times$32 with 10/100 classes.
  \item \textbf{ImageNet-1K} \cite{deng2009imagenet} contains 1.28M training and 50k validation images of size 224$\times$224 with 1,000 classes.
\end{itemize}

To assess architectural generalization, we employ six distinct neural network architectures:
\begin{itemize}
  \item \textbf{MLP} \cite{rumelhart1987learning}: A multilayer perceptron (MLP) is a fully-connected feed-forward neural network. In our experiments, we use an MLP consisting of six layers, with the number of hidden neurons in each layer set to 2048, 1024, 512, 256, 128, and 64, respectively. The final layer has 10 output neurons corresponding to the 10 prediction classes.
  \item \textbf{LeNet} \cite{lecun1998gradient}: A classic convolutional neural network (CNN) architecture designed for digit recognition, consisting of two convolutional layers followed by two fully-connected layers.
  \item \textbf{AlexNet} \cite{krizhevsky2012imagenet}: A pioneering deep CNN architecture that won the ImageNet Large Scale Visual Recognition Challenge (ILSVRC) in 2012, consisting of five convolutional layers and three fully-connected layers.
  \item \textbf{VGG series} \cite{simonyan2014very}: A deep CNN architecture characterized by its stack of small 3$\times$3 convolutional filters and a large number of layers, typically consisting of 11, 16 or 19 layers. We use VGG11 in our experiments.
  \item \textbf{ResNet series} \cite{he2016deep}: A deep residual network architecture that introduced skip connections to alleviate the vanishing gradient problem, allowing for the training of very deep networks. We use ResNet18 for our experiments.
  \item \textbf{SmallViT} \cite{dosovitskiy2021image}: A compact Vision Transformer for small images. It first projects non-overlapping image patches into tokens, prepends a learnable class token, processes the sequence with Transformer encoder blocks, and uses a two-layer MLP classification head.
\end{itemize}

All these models are implemented using PyTorch \cite{paszke2019pytorch}, which provides a flexible and efficient platform for building and training deep learning models.

\begin{figure*}[!t]
  \centering
  \begin{subfigure}[b]{0.95\textwidth}
    \centering
    \includegraphics[width=\textwidth]{figures/distribution/resnet18_cifar100_input_feats.png}
    \caption{Input feature distributions of the FC layer in ResNet18 ($B = 512$). The shown dimensions are randomly selected, and the values in each dimension approximately follow a Gaussian distribution whose mean can be estimated from similar training samples.}
  \end{subfigure}
  \hfill
  \begin{subfigure}[b]{0.95\textwidth}
    \centering
    \includegraphics[width=\textwidth]{figures/distribution/resnet18_cifar100_output_probs.png}
    \caption{Softmax probability distributions of ResNet18 model ($B = 512$). Five 100-fold classes are chosen for illustration, and the probabilities of different classes exhibit similar Gaussian distributions, with mean values close to 0.01, across various training samples.}
  \end{subfigure}
  \caption{The distributions of input features and output probabilities of the last FC layer in ResNet18, trained on CIFAR-100. The x-axis represents the values of features/probabilities, and the y-axis represents the number of samples.}
  \label{fig:feat_prob_dist}
\end{figure*}

\subsubsection{Implementation Details}
\label{sec:impl_details}

We evaluate GDBR in the FL setting consisting of $N$ clients, among which one client is designated as the victim for the purpose of label distribution inference. Each client possesses a private local dataset and participates in collaborative model training using the FedSGD algorithm~\cite{mcmahan2017communication}. In each communication round, a client randomly samples a mini-batch of $B$ data points from its local dataset; unless otherwise specified, we set the batch size $B$ to 64.

For experiments involving ResNet series, we adjust the architecture by replacing the global average pooling layer with a depthwise convolution, which has the same kernel size as the input feature map and a stride of one. Such modification does not affect the model's performance but facilitates the application of GDBR. For ViT experiments, we use a patch size of 4, embedding dimension of 192, 6 Transformer blocks, 3 attention heads, and a two-layer MLP classification head with 256 hidden units and ReLU activation. The model parameters are initialized using the default initialization scheme of PyTorch. For the classification head layers, we initialize the weights by sampling from a uniform distribution in the range $[0.01, 0.2]$.

To estimate the input features $\widetilde{\mathbf{a}}^{[1]}$ of the first stack layer and the model's output probabilities $\widetilde{\mathbf{p}}$, we employ an auxiliary dataset, which could be either the validation or test set, depending on availability. From this auxiliary dataset, we randomly sample a total of 1,000 samples, ensuring a uniform distribution across all classes to avoid class imbalance. These samples are used to approximate the feature and output distributions required by GDBR.

To further examine FedAvg-style local training (i.e., multiple SGD steps between two rounds), we also evaluate GDBR under local multi-step updates in Section~\ref{sec:local_updates}. Unless otherwise specified, these experiments use a batch size of 32, local learning rates in $\{0.001, 0.002, 0.005\}$, and local update steps in $\{1, 2, 4\}$.

To ensure statistical robustness and minimize the impact of experimental randomness, we repeat each experiment 20 times and report the mean performance across these runs. All experiments are implemented using PyTorch 2.5.1 and executed on a workstation equipped with an Intel i9-10900K CPU running at 3.70\,GHz, 64\,GB of RAM, and an NVIDIA GeForce RTX 4090 GPU with 24\,GB of memory. CUDA 12.1+ is used to accelerate computations on the GPU.

\subsubsection{Baselines and Evaluation Metrics}

We compare our proposed GDBR with two state-of-the-art baselines: ZLG~\cite{geng2021towards} and LLG~\cite{wainakh2022user}. ZLG leverages the relationship between the labels and the gradient $\nabla\mathbf{W}^{[L]}$ in the final fully-connected (FC) layer to recover labels, while LLG exploits both the direction and magnitude of $\nabla\mathbf{W}^{[L]}$ to restore the labels of batch samples. However, in our attack scenario, the gradient $\nabla\mathbf{W}^{[L]}$ is not directly accessible. To ensure a fair comparison, we provide both ZLG and LLG with a partial ground-truth gradient $\nabla\hat{\mathbf{W}}^{[L]}$ for label recovery. The approaches of providing the partial gradient are detailed in Section~\ref{sec:grad_sim_init}. Additionally, we supply the same subset of auxiliary data to GDBR and the baseline methods to estimate the averaged features of $\widetilde{\mathbf{a}}^{[1]}$ and the model's output Softmax probabilities $\widetilde{\mathbf{p}}$.

To quantify the performance of label restoration, we adopt two complementary metrics: \emph{Instance-level Accuracy} (InsAcc) and \emph{Class-level Accuracy} (ClsAcc)~\cite{zhang2024posterior}. These two metrics capture different levels of label leakage. InsAcc evaluates whether the attacker recovers the batch labels with the correct multiplicities, while ClsAcc evaluates whether the attacker identifies which classes appear in the batch, regardless of how many samples each class contributes.

\begin{itemize}
  \item \textbf{Instance-level Accuracy (InsAcc)}: This metric is multiset-based and evaluates how accurately the attacker restores the individual labels in the batch. Label multiplicities are preserved. For example, if a target batch contains 10 dog images and 6 cat images, InsAcc measures how many of these 16 instance labels are correctly recovered, including the correct number of dogs and cats.
  \item \textbf{Class-level Accuracy (ClsAcc)}: This metric is set-based and evaluates whether the attacker correctly detects the presence or absence of classes in the batch. It ignores multiplicities. In the same example, ClsAcc measures whether the attacker identifies that dog and cat are present and, equally importantly, that other classes such as bird are absent. This is still challenging for label-inference attacks because the attacker must infer the batch-level class support from gradients.
\end{itemize}

Both metrics are realized through Jaccard similarity, but over different objects. For InsAcc, let $\hat{\mathbf{y}}$ and $\mathbf{y}$ denote the recovered and ground-truth batch-label multisets, respectively. The multiset intersection counts the minimum recovered and ground-truth multiplicity for each class, and the multiset union counts the maximum multiplicity for each class. InsAcc is computed as:
\begin{equation*}
  \text{InsAcc} = J(\hat{\mathbf{y}}, \mathbf{y}) = \frac{|\hat{\mathbf{y}} \cap \mathbf{y}|}{|\hat{\mathbf{y}} \cup \mathbf{y}|} = \frac{|\hat{\mathbf{y}} \cap \mathbf{y}|}{|\hat{\mathbf{y}}| + |\mathbf{y}| - |\hat{\mathbf{y}} \cap \mathbf{y}|},
\end{equation*}
where $|\hat{\mathbf{y}} \cap \mathbf{y}|$ is the number of correctly recovered instance labels under multiset matching. For ClsAcc, we first convert batch labels into class-presence sets. Let $\hat{\mathbf{c}}=\mathrm{supp}(\hat{\mathbf{y}})$ and $\mathbf{c}=\mathrm{supp}(\mathbf{y})$ denote the predicted and ground-truth sets of classes appearing in the batch. ClsAcc is then computed as:
\begin{equation*}
  \text{ClsAcc} = J(\hat{\mathbf{c}}, \mathbf{c}) = \frac{|\hat{\mathbf{c}} \cap \mathbf{c}|}{|\hat{\mathbf{c}} \cup \mathbf{c}|} = \frac{|\hat{\mathbf{c}} \cap \mathbf{c}|}{|\hat{\mathbf{c}}| + |\mathbf{c}| - |\hat{\mathbf{c}} \cap \mathbf{c}|},
\end{equation*}
where $\hat{\mathbf{c}}$ and $\mathbf{c}$ contain each class at most once. Thus, InsAcc reflects fine-grained recovery of per-instance label counts, whereas ClsAcc reflects recovery of the training batch's class support.

\begin{figure*}[!t]
  \centering
  \begin{subfigure}[b]{0.24\textwidth}
    \centering
    \includegraphics[width=\textwidth]{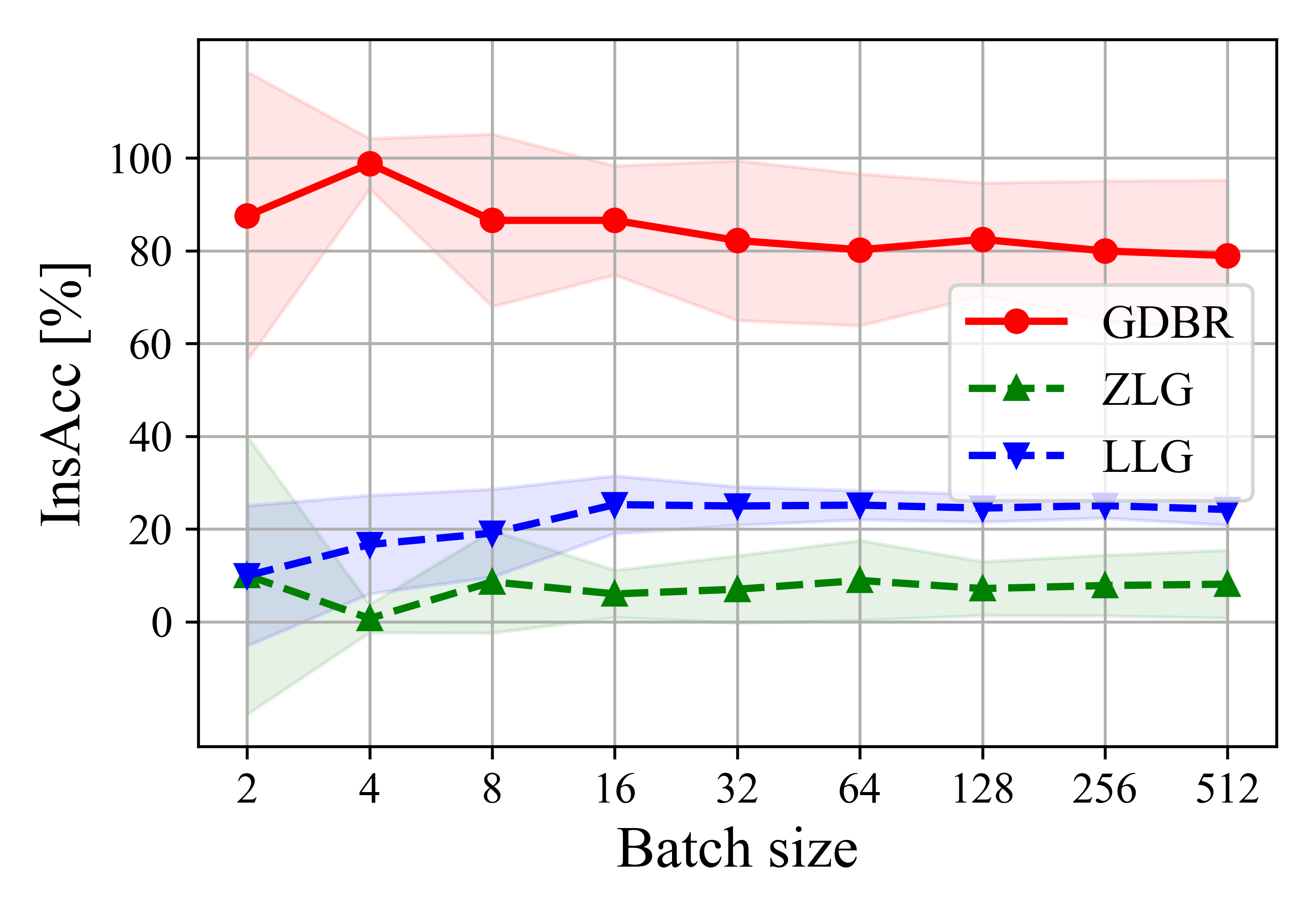}
    \caption{InsAcc on MNIST.}
  \end{subfigure}
  \hfill
  \begin{subfigure}[b]{0.24\textwidth}
    \centering
    \includegraphics[width=\textwidth]{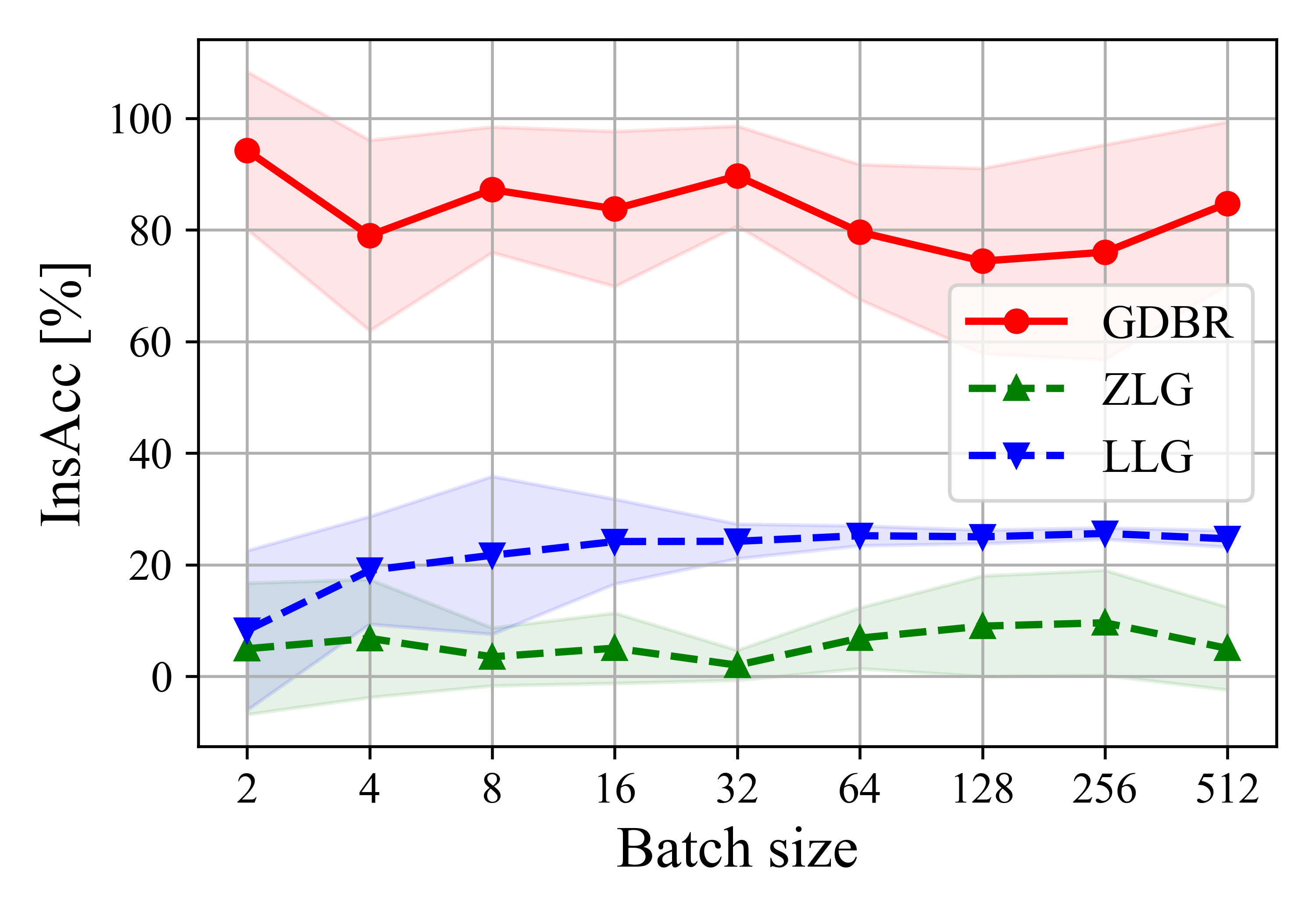}
    \caption{InsAcc on SVHN.}
  \end{subfigure}
  \hfill
  \begin{subfigure}[b]{0.24\textwidth}
    \centering
    \includegraphics[width=\textwidth]{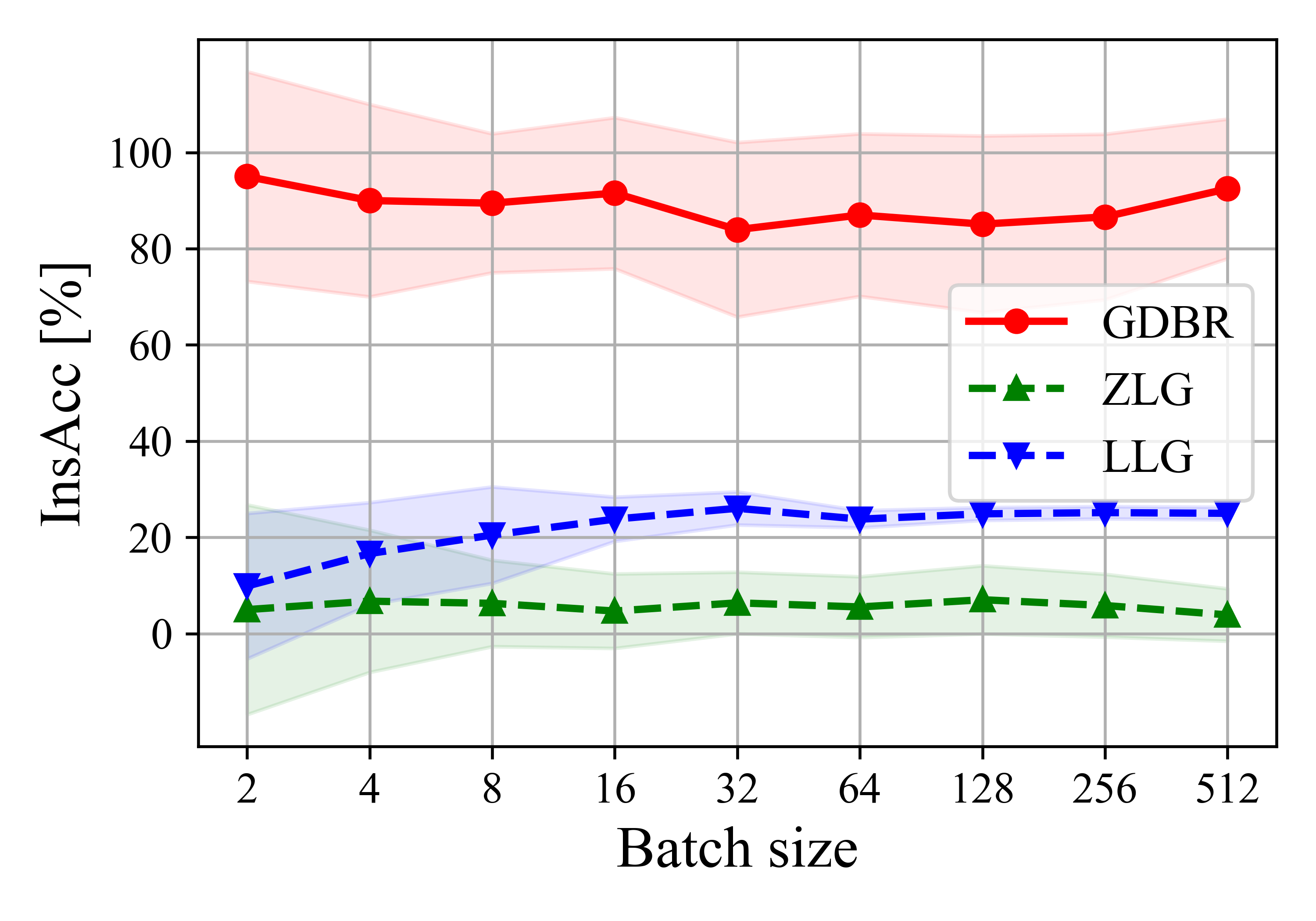}
    \caption{InsAcc on CIFAR-10.}
  \end{subfigure}
  \hfill
  \begin{subfigure}[b]{0.24\textwidth}
    \centering
    \includegraphics[width=\textwidth]{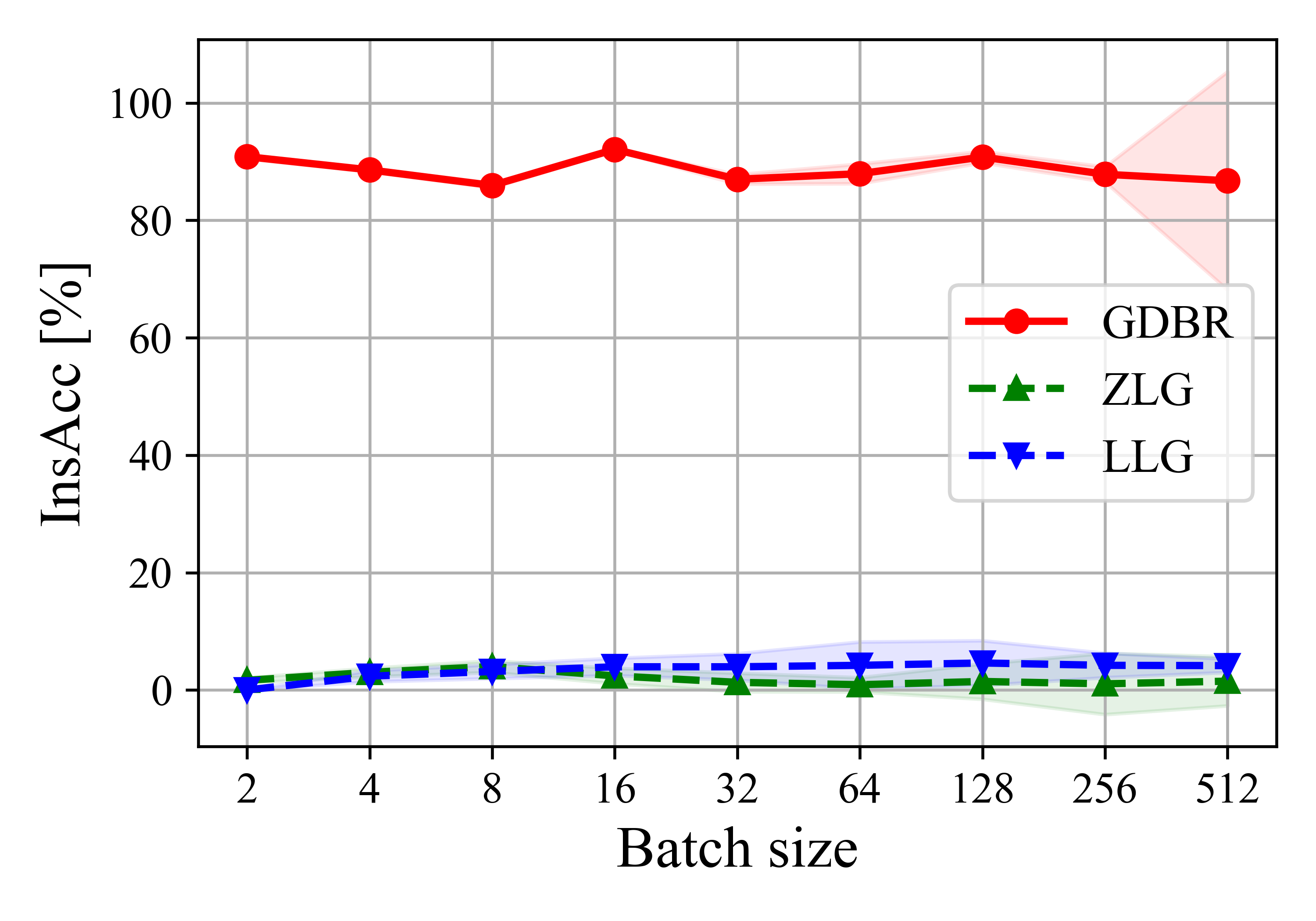}
    \caption{InsAcc on CIFAR-100.}
  \end{subfigure}
  \hfill
  \begin{subfigure}[b]{0.24\textwidth}
    \centering
    \includegraphics[width=\textwidth]{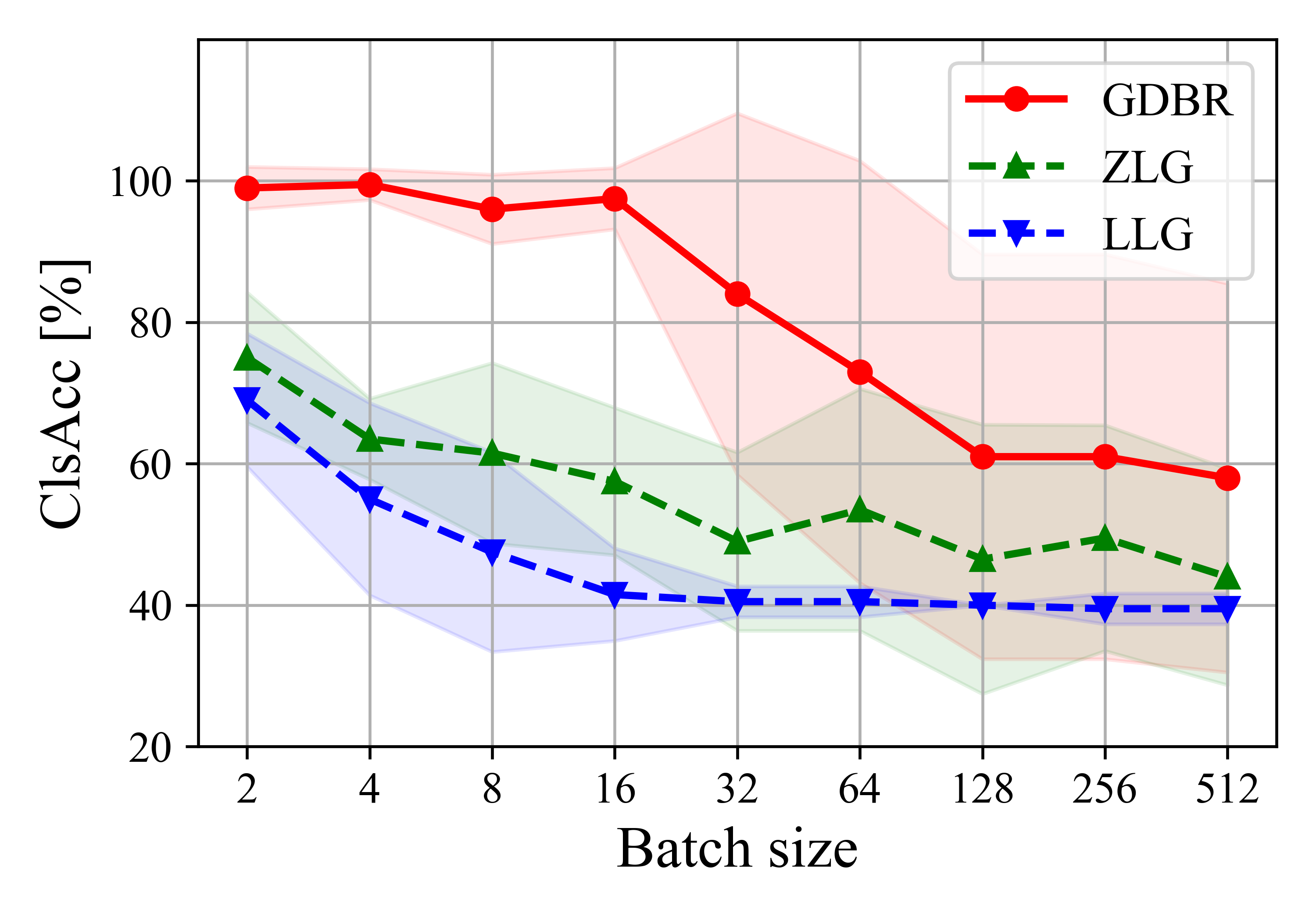}
    \caption{ClsAcc on MNIST.}
  \end{subfigure}
  \hfill
  \begin{subfigure}[b]{0.24\textwidth}
    \centering
    \includegraphics[width=\textwidth]{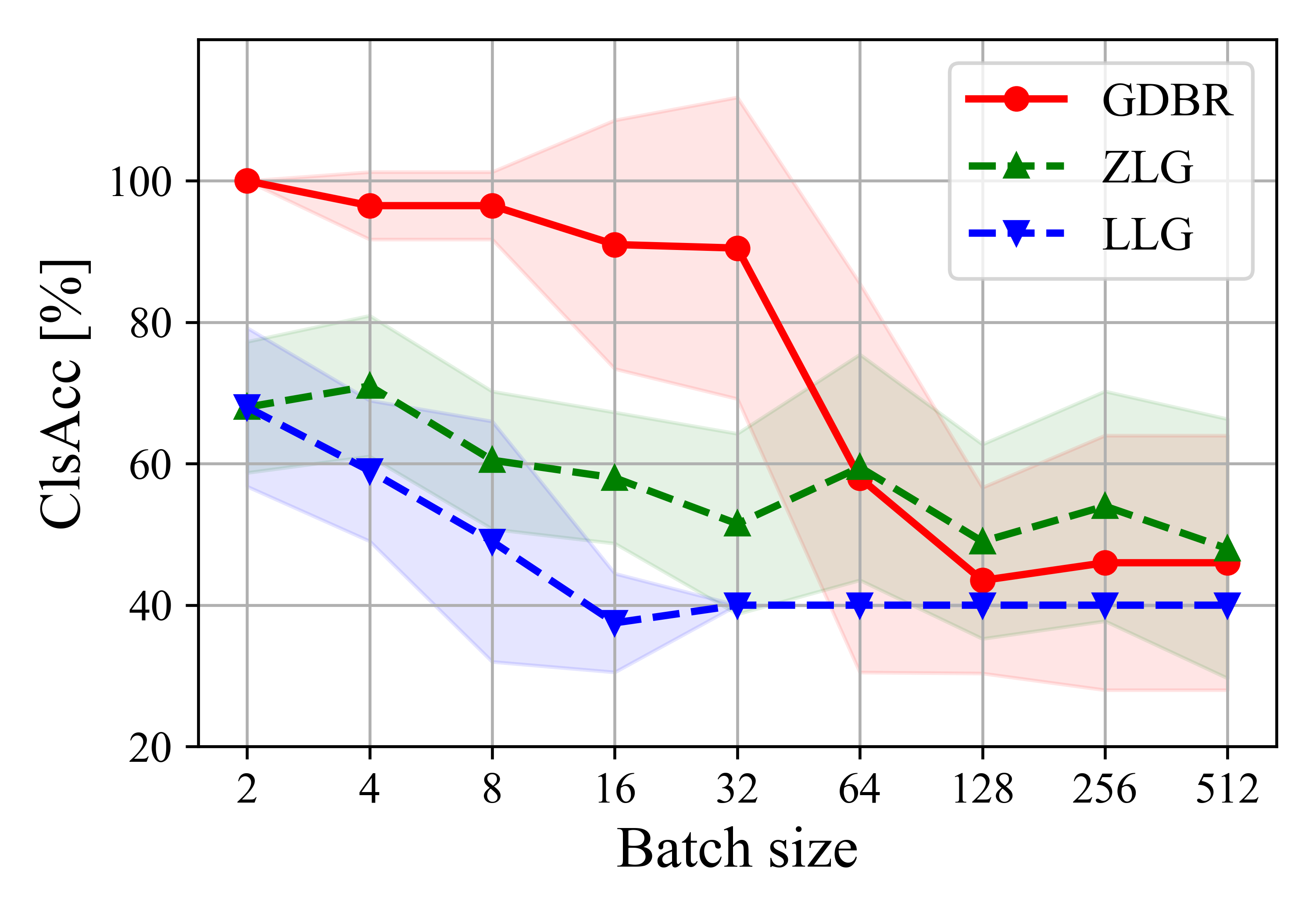}
    \caption{ClsAcc on SVHN.}
  \end{subfigure}
  \hfill
  \begin{subfigure}[b]{0.24\textwidth}
    \centering
    \includegraphics[width=\textwidth]{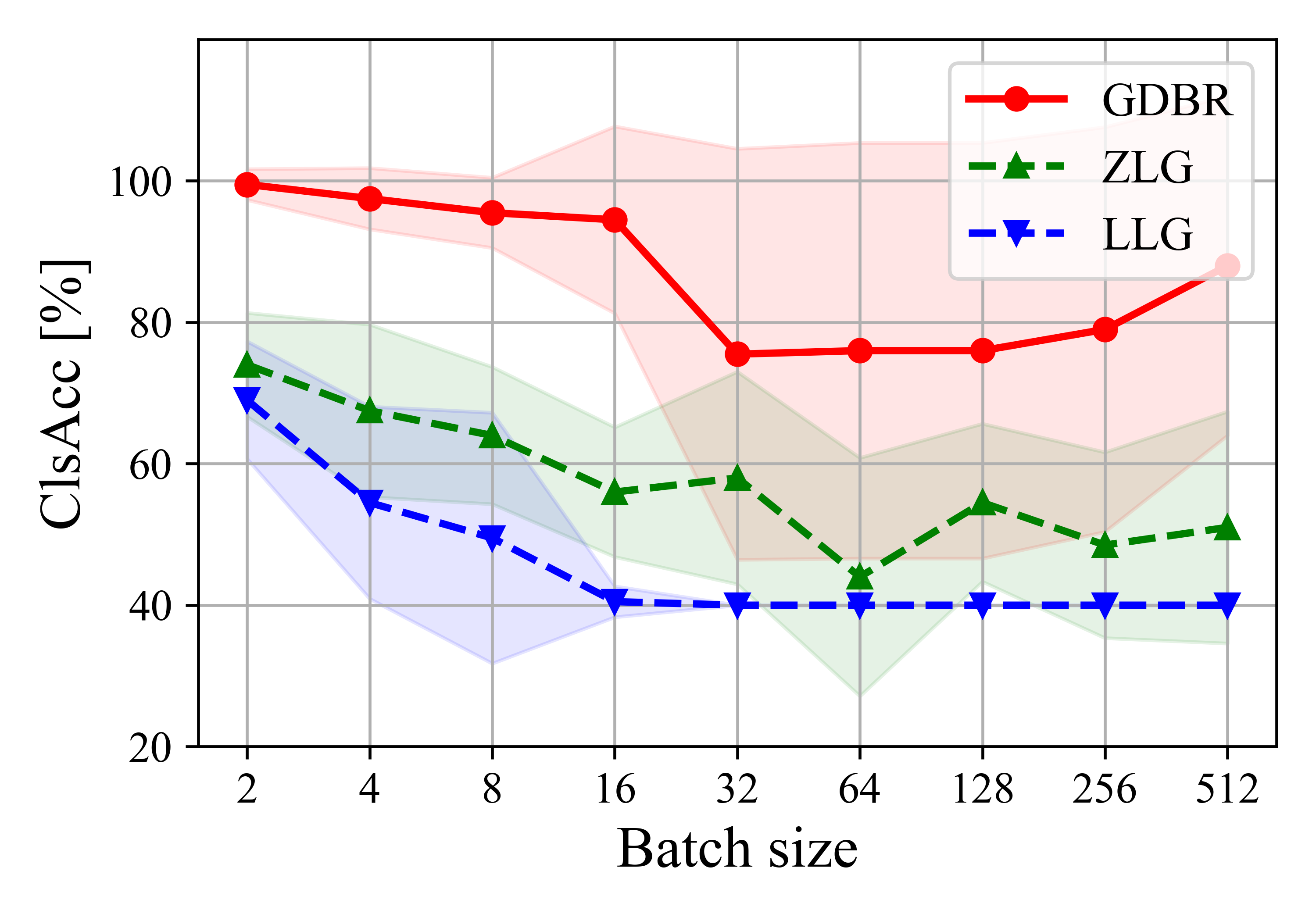}
    \caption{ClsAcc on CIFAR-10.}
  \end{subfigure}
  \hfill
  \begin{subfigure}[b]{0.24\textwidth}
    \centering
    \includegraphics[width=\textwidth]{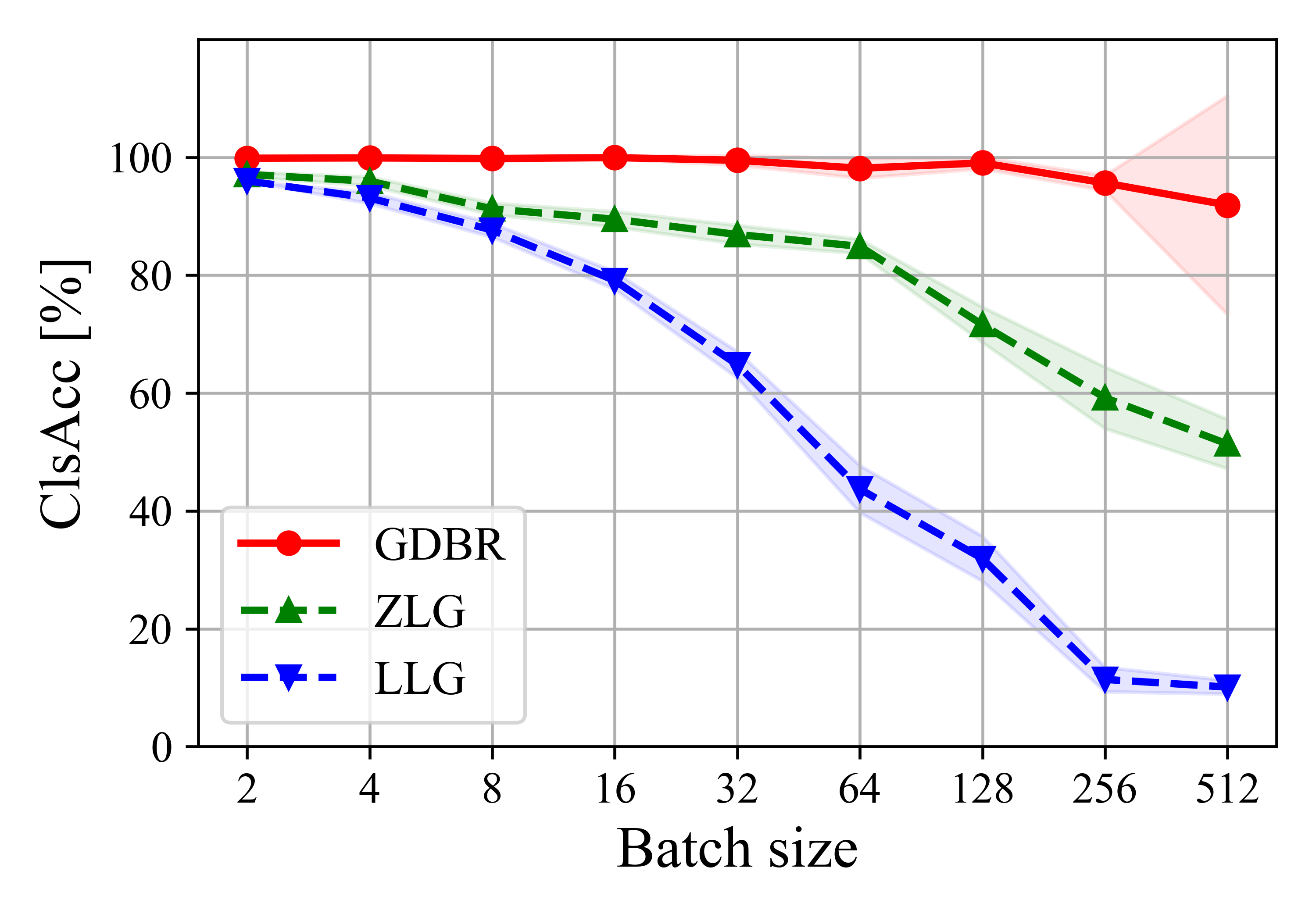}
    \caption{ClsAcc on CIFAR-100.}
  \end{subfigure}
  \caption{Comparison of GDBR with baselines on InsAcc and ClsAcc across various batch sizes. Experiments are conducted on four model-dataset pairs: LeNet on MNIST, AlexNet on SVHN, VGG11 on CIFAR-10, and ResNet18 on CIFAR-100, with batch sizes ranging from 2 to 512. Results are averaged over 20 trials (line and dots); shaded area shows standard deviation.}
  \label{fig:batch_size_comp}
\end{figure*}

\subsection{Verification of Assumptions}
\label{sec:verify_assumptions}

In Section~\ref{sec:verify_assumptions}, we have presented the assumptions of the homogeneity of features and probabilities of the classification head layer's input and output within a batch. Hence, we first validate these assumptions through empirical analysis. We conduct experiments using ResNet18 trained on the CIFAR-100 dataset with a batch size of $B = 512$.

Specifically, we randomly select five feature dimensions (23, 91, 230, 271, and 477) and five classes (10, 30, 50, 70, and 90) to illustrate the distributions of the extracted features and output probabilities from the final FC layer of classification head. The results of these visualizations are presented in Fig.~\ref{fig:feat_prob_dist}.

The observed distributions of input features and output probabilities across different samples are highly similar, regardless of the specific dimension or index selected. This consistency indicates that the features and probabilities can be effectively approximated using auxiliary data.

\subsection{Comparison With Baselines}

We compare the proposed GDBR method with representative baselines across a diverse set of model-dataset configurations, aiming to evaluate label recovery performance under varying conditions. 
The tested combinations include: LeNet on MNIST, AlexNet on SVHN, VGG11 on CIFAR-10, and ResNet18 on CIFAR-100, covering both simple and complex architectures as well as datasets of different visual characteristics. 
Batch sizes range from $2$ to $512$, with each batch randomly sampled from a subset of all available classes in the training set to simulate realistic, heterogeneous label distributions.

For a fair comparison, we provide ZLG and LLG with the weight gradient $\nabla\mathbf{W}^{[L]}$ of the final layer, but with the column elements shuffled to prevent direct label inference while still retaining comparable gradient information. 
This setting allows us to test the robustness of the baselines when valuable yet perturbed information is available.

As shown in Fig.~\ref{fig:batch_size_comp}, GDBR consistently outperforms ZLG and LLG in both InsAcc and ClsAcc across almost all evaluated settings. 
Even when the baselines are supplied with enriched gradient information, they struggle to correctly recover the labels of the target batch. 
In contrast, GDBR achieves notably higher recovery accuracies, demonstrating its effectiveness and adaptability across different architectures, datasets, and batch configurations.

\subsection{Analysis of Different Factors}

In this subsection, we analyze how several factors affect the performance of GDBR, including gradients from different layers, the class distribution within batch training data, and the use of auxiliary versus dummy data.

\subsubsection{Effect of Gradients From Different Layers}

We examine how the provided gradients $\overline{\nabla\mathbf{W}}^{[1]}$ from different FC layers affects the performance of GDBR. We conduct the experiments on the 6-layer MLP, LeNet, AlexNet, and VGG11 models, which contain multiple FC layers. For the MLP model, we vary the layer from which the gradient is extracted, ranging from the 2nd to the 6th layer (the input layer). For LeNet, AlexNet, and VGG11, we consider gradients from all FC layers, including both intermediate and output layers. The models are trained on SVHN and CIFAR-10 datasets with a batch size of $B = 64$.

\begin{figure}[!b]
  \centering
  \begin{subfigure}[b]{0.232\textwidth}
    \centering
    \includegraphics[width=\textwidth]{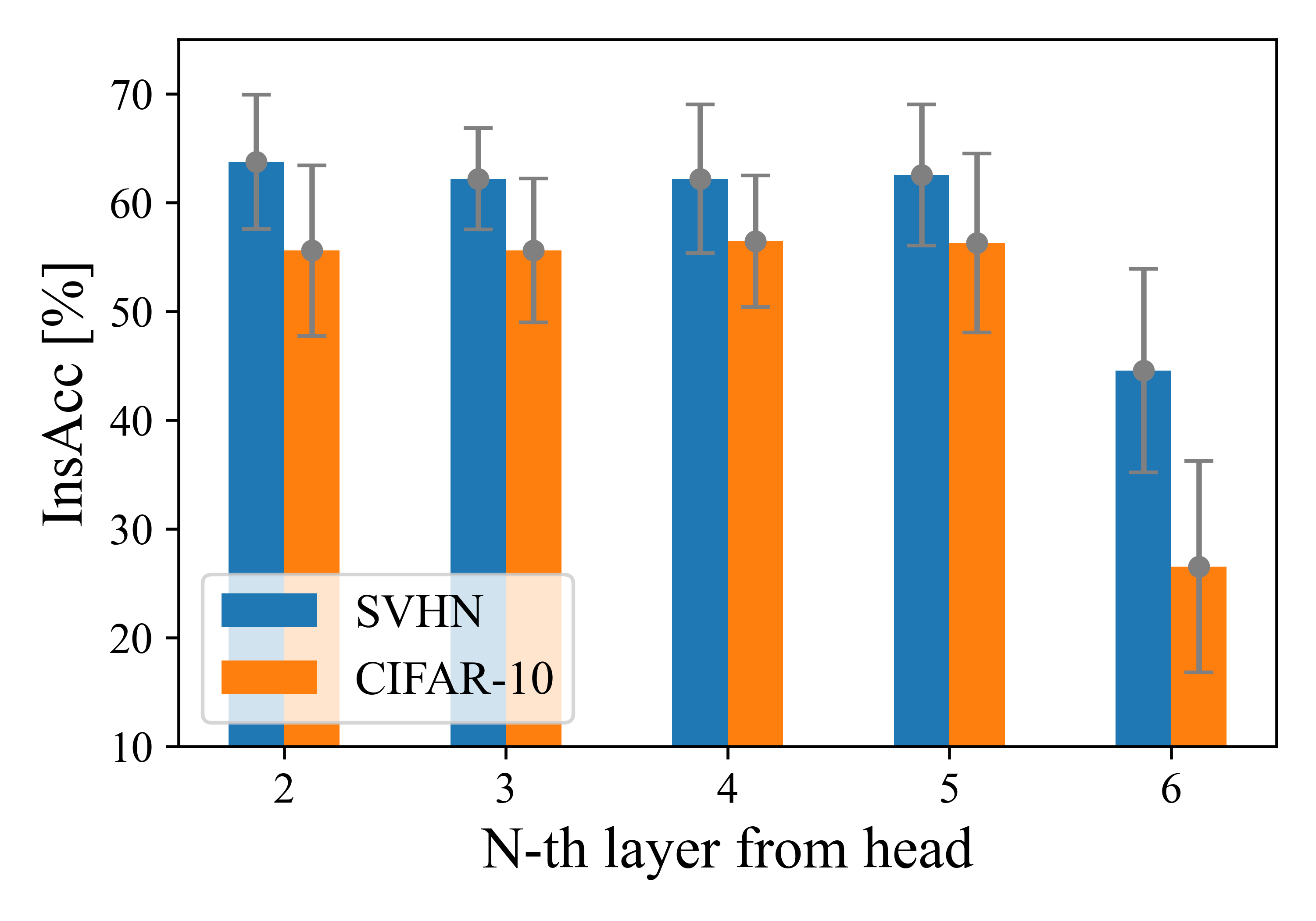}
    \caption{InsAcc on 6-layer MLP.}
    \label{fig:mlp_insacc}
  \end{subfigure}
  \hfill
  \begin{subfigure}[b]{0.232\textwidth}
    \centering
    \includegraphics[width=\textwidth]{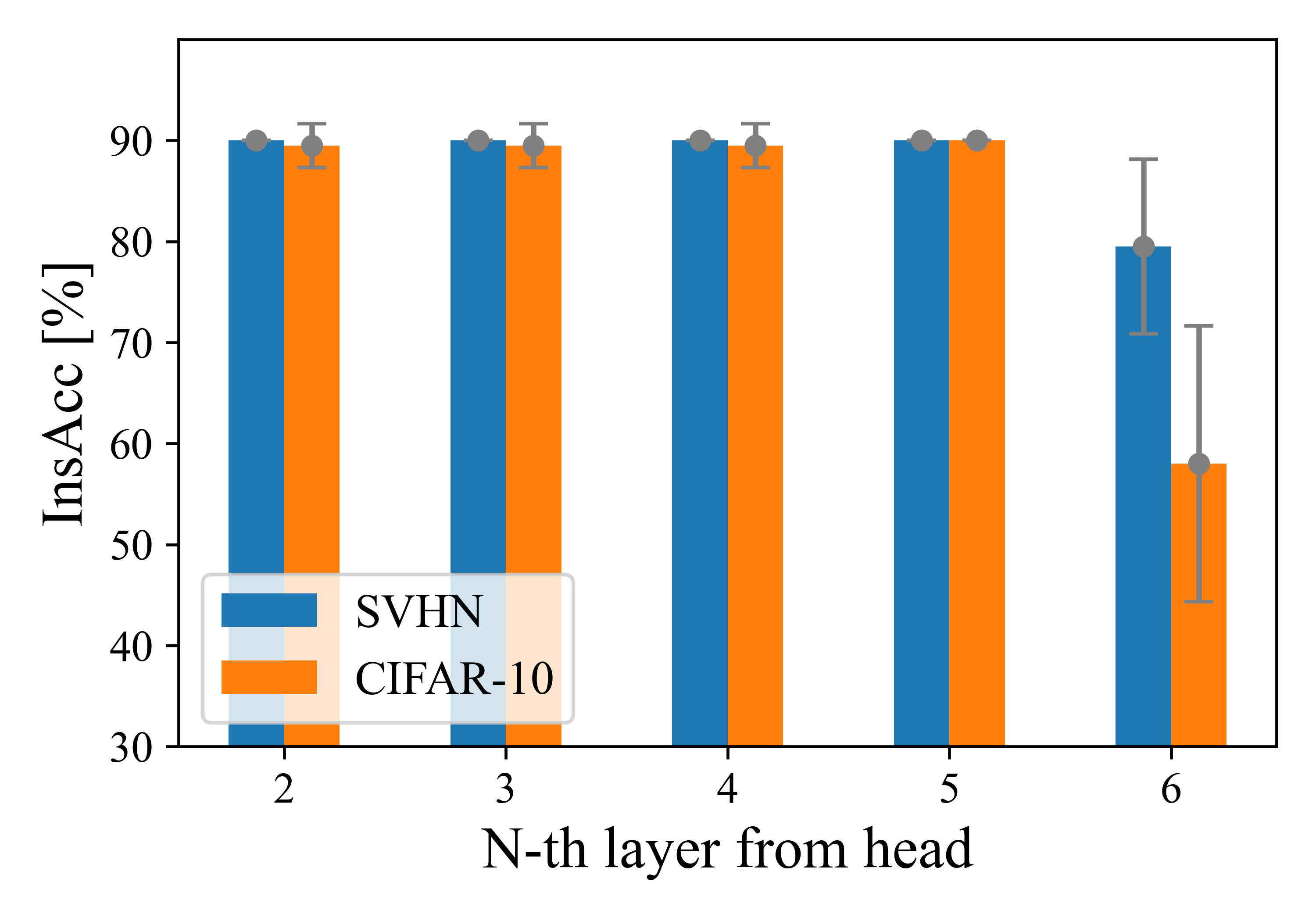}
    \caption{ClsAcc on 6-layer MLP.}
    \label{fig:mlp_clsacc}
  \end{subfigure}
  \hfill
  \begin{subfigure}[b]{0.232\textwidth}
    \centering
    \includegraphics[width=\textwidth]{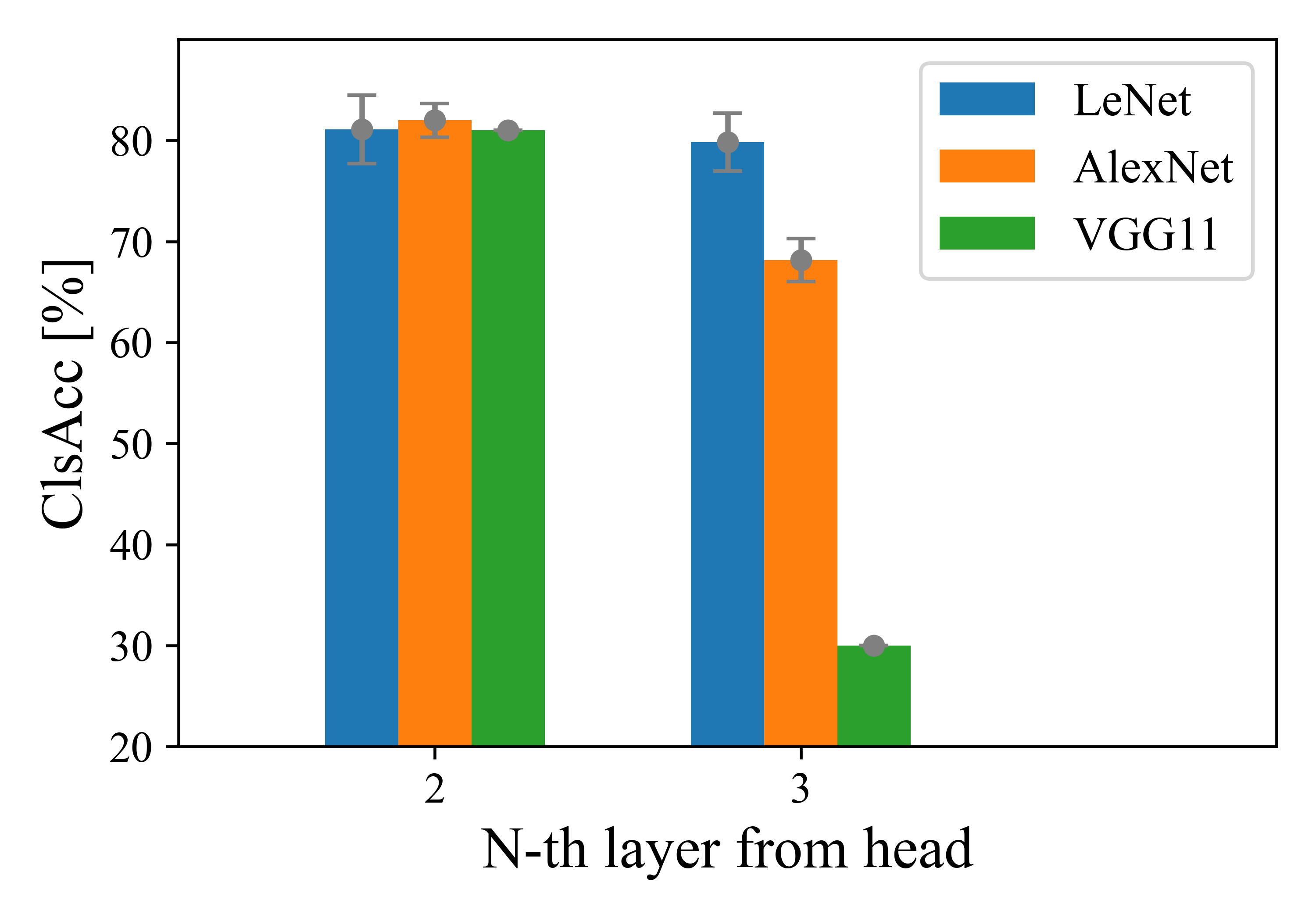}
    \caption{InsAcc on CNN models.}
    \label{fig:cnn_insacc}
  \end{subfigure}
  \hfill
  \begin{subfigure}[b]{0.232\textwidth}
    \centering
    \includegraphics[width=\textwidth]{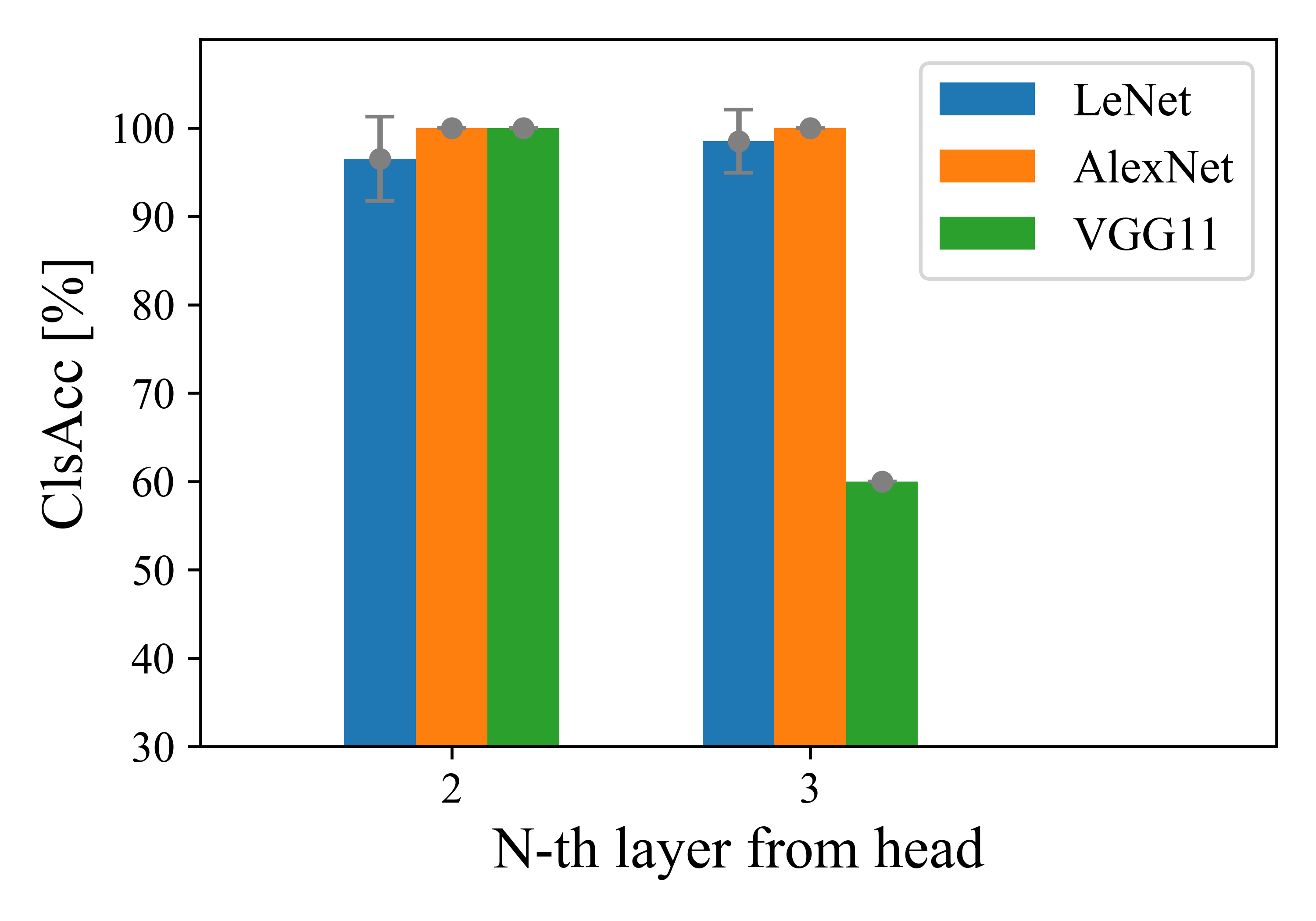}
    \caption{ClsAcc on CNN models.}
    \label{fig:cnn_clsacc}
  \end{subfigure}
  \caption{Comparison of the shared gradients $\overline{\nabla\mathbf{W}}^{[1]}$ from different FC layers in a 6-layer MLP, LeNet, AlexNet, and VGG11. Experiments are conducted on the models trained on SVHN and CIFAR-10 with batch size $B = 64$. The x-axis indicates the layer distance to the classification head's output layer, where 1 means the final output layer. For MLP, the 6th layer corresponds to the input layer; while for CNNs, the 3rd layer corresponds to the first FC layer after the convolutional blocks.}
  \label{fig:layer_effect}
\end{figure}

For the 6-layer MLP, the results are demonstrated in Fig.~\ref{fig:mlp_insacc} and Fig.~\ref{fig:mlp_clsacc}. It is observed that gradients closer to the output layer yield better label recovery performance. This is because features in the deeper layers are more concentrated and semantically rich, making it easier to estimate them using auxiliary data and to construct an accurate gradient bridge to the output logits. Conversely, gradients from layers closer to the input yield suboptimal results, as the primitive and highly dispersed nature of low-level features poses significant challenges for accurate estimation.

For CNN models, the results are shown in Fig.~\ref{fig:cnn_insacc} and Fig.~\ref{fig:cnn_clsacc}. We find that gradients from different FC layers lead to similar performance in terms of InsAcc and ClsAcc. Moreover, we observe that the accuracies of VGG11 decreases dramatically when using gradients from the first FC layer. This is likely because the input features to the first FC layer in VGG11 are high-dimensional and sparse, making it challenging to estimate them accurately using auxiliary data. Overall, these findings demonstrate that GDBR can effectively utilize gradients from various layers for label recovery, with a preference for deeper layers in MLPs and CNNs.

\subsubsection{Effect of Batch Class Distribution}

We also conduct an in-depth analysis of how varying batch class distributions influence the performance of GDBR. Specifically, we examine five distinct distribution settings: \emph{random}, \emph{uniform}, \emph{single}, \emph{subclassed}, and \emph{imbalanced}. For each setting, we mainly measure the InsAcc of label recovery and present the results in Table~\ref{tab:batch_dist}.

Overall, GDBR demonstrates consistently strong performance across all distribution types, indicating that its effectiveness is not overly sensitive to the composition of the batches. Although variations in the data distribution lead to slight differences in accuracy, the method maintains a high level of reliability in recovering the labels of the target batch. Interestingly, batches that contain samples from only a single class tend to produce the highest accuracy. This improvement is likely due to the concentration of feature patterns within the batch, which reduces intra-batch variability and facilitates more precise estimation of the underlying variables. These observations provide further evidence of the robustness of GDBR, even when faced with diverse and potentially challenging batch composition scenarios.

\begin{table}[htbp]
  \centering
  \caption{Comparison of GDBR's label recovery performance using different class distributions within the training data across various datasets and models combinations.}
  \label{tab:batch_dist}

  \begin{threeparttable}
    \setlength{\tabcolsep}{4.8pt}
    \renewcommand{\arraystretch}{1.35}
    \centering
    \begin{adjustbox}{width=0.96\columnwidth}
      \begin{tabular}{cccccccccccc}
        \specialrule{0.85pt}{2.5pt}{0.35pt}
        \multirow{2}{*}{\textbf{Dataset}} & \multirow{2}{*}{\textbf{Model}} & \multicolumn{5}{c}{\textbf{InsAcc [\%] $\uparrow$}} \\ \cline{3-7}
        &      & R*    & U     & S     & SC    & IM \\
        \specialrule{0.75pt}{0.35pt}{0.35pt}
        MNIST       & LeNet       & 80.7  & 83.0  & 86.3  & 85.0  & 93.5   \\
        SVHN        & AlexNet     & 85.6  & 81.6  & 81.7  & 85.2  & 90.1   \\
        CIFAR-10    & VGG11       & 84.1  & 81.2  & 89.6  & 90.0  & 97.0   \\
        CIFAR-100   & ResNet18    & 89.8  & 88.5  & 94.3  & 87.9  & 90.4   \\
        ImageNet    & ResNet50    & 95.2  & 97.0  & 98.8  & 90.0  & 91.0  \\
        \specialrule{0.85pt}{0.35pt}{0pt}
      \end{tabular}
    \end{adjustbox}

    \begin{tablenotes}
      \small
      \item *\underline{R}andom, \underline{U}niform, \underline{S}ingle, \underline{S}ub\underline{C}lassed, \underline{IM}balanced.
    \end{tablenotes}
  \end{threeparttable}
\end{table}

\subsubsection{Effect of Auxiliary Data vs. Dummy Data}

We further examine the effect of employing auxiliary data as opposed to dummy data in the GDBR framework. As outlined in the implementation details, auxiliary data is sampled from the validation or test set, consisting of 1,000 images that are evenly distributed across all classes, thereby ensuring balanced class coverage. In contrast, dummy data is synthetically generated from a standard Gaussian distribution, without any semantic correspondence to the actual dataset.

The comparative evaluation results, summarized in Table~\ref{tab:aux_dummy}, reveal an interesting trend. For relatively simple datasets such as MNIST, which contain low-resolution grayscale images with limited variability, GDBR can achieve even higher InsAcc when using dummy data. This suggests that, in such cases, the neural network's feature space can be effectively approximated using purely random inputs, making feature estimation sufficiently accurate for label recovery.

However, for more complex datasets involving high-resolution color images and greater intra-class diversity (e.g., CIFAR-100, ImageNet), GDBR consistently performs better when auxiliary data is used. The real auxiliary samples are drawn from the same distribution as the target data. This consistency ensures more representative feature activations and output probability patterns, thereby leading to a more precise estimation of the variables required for label reconstruction.

Importantly, the viability of GDBR extends to scenarios where auxiliary data is unavailable. Our findings reveal that dummy data achieves competitive performance, although a slight performance trade-off is observed with more complex datasets. This highlights the adaptability of GDBR and its potential applicability in situations where access to real auxiliary samples is restricted.

\begin{table}[!t]
  \centering
  \caption{Comparison of GDBR's label recovery performance using auxiliary data vs. dummy data across various datasets and models combinations.}
  \label{tab:aux_dummy}

  \begin{threeparttable}
    \setlength{\tabcolsep}{4.8pt}
    \renewcommand{\arraystretch}{1.4}
    \begin{adjustbox}{width=0.96\columnwidth}
      \centering
      \begin{tabular}{cccccccccccc}
        \specialrule{0.85pt}{2.5pt}{0.35pt}
        \multirow{2}{*}{\textbf{Dataset}} & \multirow{2}{*}{\textbf{Model}} & \multicolumn{2}{c}{\textbf{InsAcc [\%] $\uparrow$}} & \multicolumn{2}{c}{\textbf{ClsAcc [\%] $\uparrow$}} \\ \cline{3-6}
        &           & Aux*        & Dummy          & Aux            & Dummy   \\
        \specialrule{0.75pt}{0.35pt}{0.35pt}
        MNIST       & LeNet      & 81.{\fontsize{6}{5}\selectfont $\pm$5.}   & 83.{\fontsize{6}{5}\selectfont $\pm$6.}   & 98.{\fontsize{6}{5}\selectfont $\pm$4.}   & 95.{\fontsize{6}{5}\selectfont $\pm$5.}   \\
        SVHN        & AlexNet    & 86.{\fontsize{6}{5}\selectfont $\pm$5.}   & 80.{\fontsize{6}{5}\selectfont $\pm$4.}   & 99.{\fontsize{6}{5}\selectfont $\pm$3.}   & 98.{\fontsize{6}{5}\selectfont $\pm$4.}   \\
        CIFAR-10    & VGG11      & 84.{\fontsize{6}{5}\selectfont $\pm$6.}   & 85.{\fontsize{6}{5}\selectfont $\pm$7.}   & 97.{\fontsize{6}{5}\selectfont $\pm$5.}   & 99.{\fontsize{6}{5}\selectfont $\pm$2.}    \\
        CIFAR-100   & ResNet18   & 90.{\fontsize{6}{5}\selectfont $\pm$4.}   & 85.{\fontsize{6}{5}\selectfont $\pm$3.}   & 98.{\fontsize{6}{5}\selectfont $\pm$2.}   & 97.{\fontsize{6}{5}\selectfont $\pm$2.}    \\
        ImageNet    & ResNet50   & 95.{\fontsize{6}{5}\selectfont $\pm$2.}   & 92.{\fontsize{6}{5}\selectfont $\pm$2.}   & 99.{\fontsize{6}{5}\selectfont $\pm$1.}   & 98.{\fontsize{6}{5}\selectfont $\pm$2.}    \\
        \specialrule{0.85pt}{0.35pt}{0pt}
      \end{tabular}
    \end{adjustbox}
  \end{threeparttable}

  \begin{tablenotes}
    \small
    \item *Aux: auxiliary data, Dummy: dummy data.
  \end{tablenotes}
\end{table}

\subsubsection{Effect of Gradient Simulation for Baselines}
\label{sec:grad_sim_init}

We further investigate the impact of different gradient simulation strategies for the baselines. The simulation methods are designed to provide ZLG and LLG with partial or perturbed ground-truth gradients $\nabla{\mathbf{W}}^{[L]}$ of the classification output layer, thereby emulating scenarios where complete and accurate gradients are unavailable. Specifically, we consider the following settings:

\begin{itemize}
  \item \textbf{Stats}: Simulate gradients using global statistical properties (e.g., mean and variance) of the full ground-truth gradient $\nabla\mathbf{W}^{[L]}$, thus retaining coarse distributional information without exact element values.
  \item \textbf{Stats (col)}: Simulate gradients using column-wise statistics of $\nabla\mathbf{W}^{[L]}$, preserving per-class statistical characteristics while obscuring fine-grained element-level details.
  \item \textbf{Shuffle (col)}: Simulate gradients by shuffling elements within each column of $\nabla\mathbf{W}^{[L]}$, maintaining marginal distributions for each class but disrupting positional correlations.
  \item \textbf{Mask}: Randomly mask $50\%$ of the elements in $\nabla\mathbf{W}^{[L]}$ and replace them with the mean value of the remaining elements, effectively reducing the granularity of the available gradient information.
\end{itemize}

These simulation strategies allow us to systematically assess the robustness of GDBR and the sensitivity of baseline methods to incomplete or perturbed gradient signals.

\begin{figure}[htbp]
  \centering
  \includegraphics[width=0.75\columnwidth]{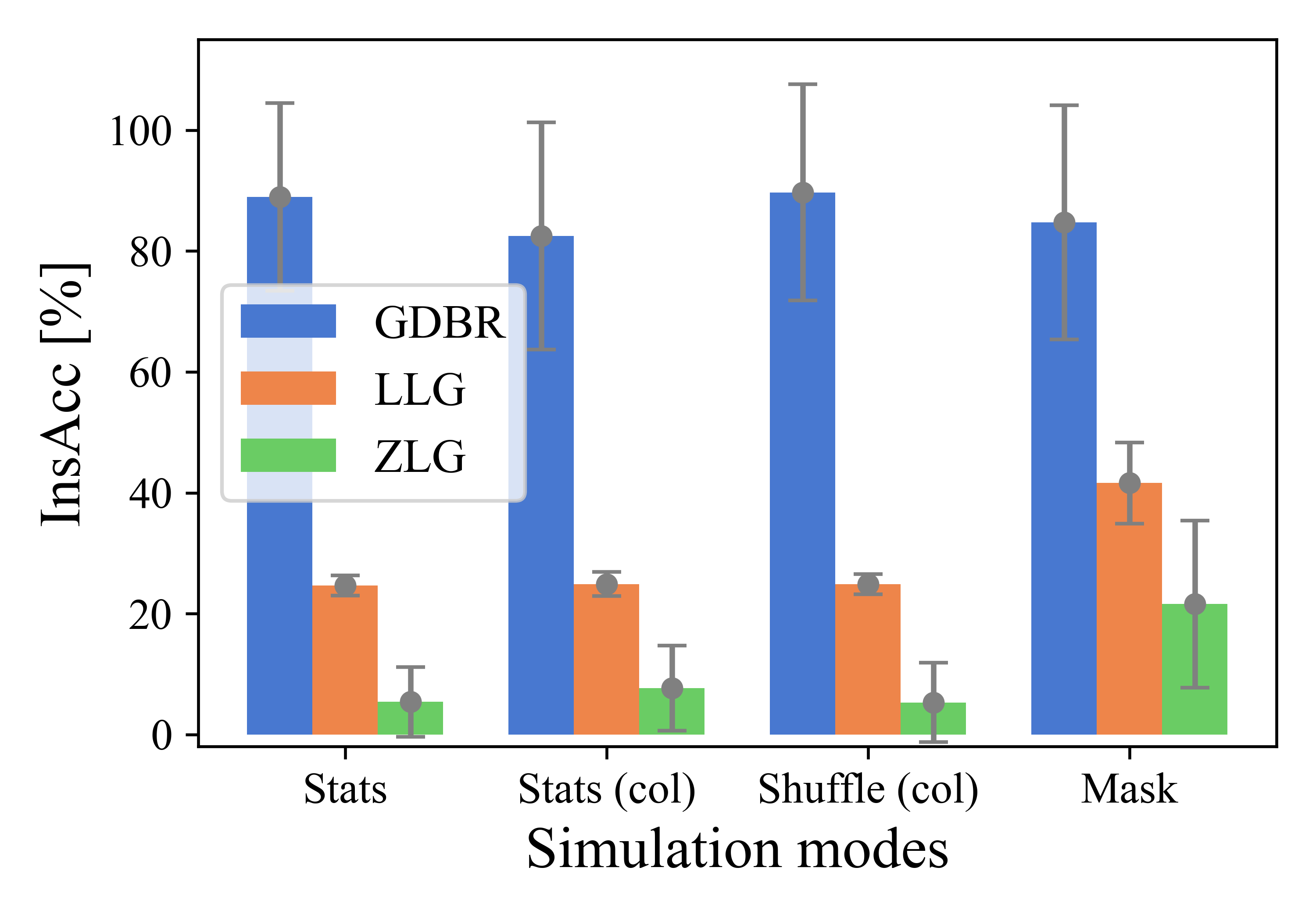}
  \caption{Comparison of gradient simulation strategies for ZLG and LLG baselines. The experiment is conducted on VGG11 trained on CIFAR-10 using auxiliary data.}
  \label{fig:grad_sim_modes}
\end{figure}

As shown in Fig.~\ref{fig:grad_sim_modes}, although the \emph{Mask} strategy achieves marginally better performance than other simulation modes, the proposed GDBR consistently surpasses both ZLG and LLG by a significant margin. Notably, GDBR maintains this superiority without requiring any ground-truth gradient information, even when the baselines are advantaged with partial access to it. This underscores the exceptional efficacy and robustness of GDBR, which manages to restore labels more accurately than baselines that are provided with superior information.

\subsubsection{Effect of Model Initialization Modes}

We further examine the influence of different model initialization modes on GDBR's performance. As outlined in the implementation details, our positive initialization scheme sets the weights of the classification head layers using a uniform distribution in the range $[0.01, 0.2]$, ensuring that all estimated feature values are positive. This design choice aims to simplify downstream feature estimation and reduce sign ambiguity in gradient interpretation. In contrast, the default PyTorch initialization adopts Kaiming uniform~\cite{he2015delving}, producing values drawn from both negative and positive ranges, which more closely reflects typical training scenarios. Additionally, we replace zero-valued features with the mean of the nonzero elements to avoid potential degeneracy in subsequent computations.

\begin{figure}[htbp]
  \centering
  \includegraphics[width=0.75\columnwidth]{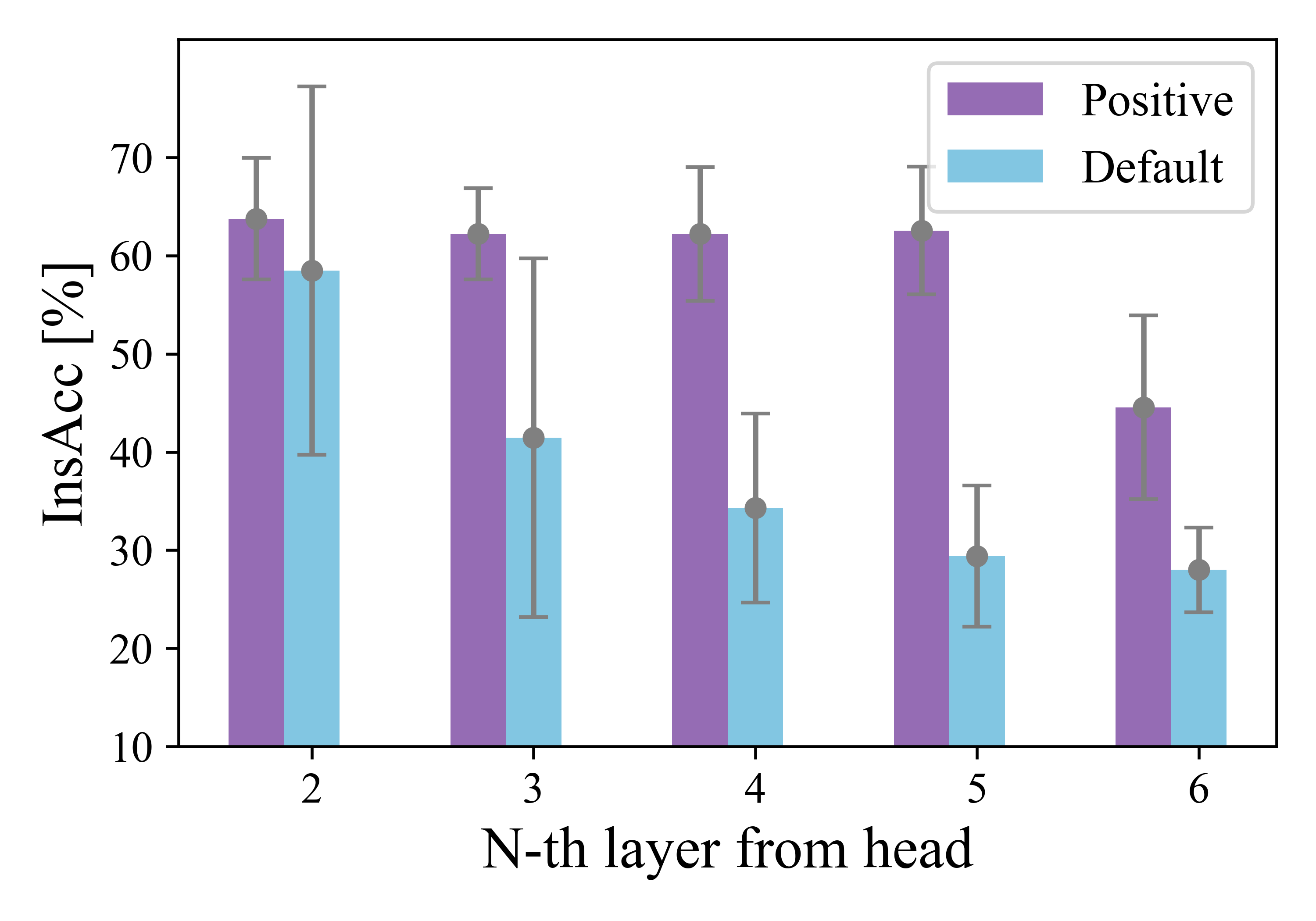}
  \caption{Impact of model initialization modes on GDBR. The experiment uses a 6-layer MLP trained on SVHN and evaluates gradients from different layers.}
  \label{fig:init_modes}
\end{figure}

Using a 6-layer MLP trained on SVHN, we compare GDBR's performance under these two different initialization schemes, utilizing gradients from various layers to evaluate sensitivity to initialization. As demonstrated in Fig.~\ref{fig:init_modes}, GDBR achieves higher accuracy under positive initialization, suggesting that constraining the representation feature signs can provide a modest benefit. However, the observed performance gap remains narrow, indicating that the choice of layer depth plays a far more pivotal role than the initialization scheme. Specifically, even without positive initialization, utilizing gradients from the penultimate FC layer still enables superior label inference, underscoring GDBR's robustness across diverse weight distributions.

\subsection{Generalization to Vision Transformers}
\label{sec:vit_generalization}

To examine whether the leakage exploited by GDBR is limited to conventional MLP and CNN architectures, we further evaluate GDBR on SmallViT using MNIST, SVHN, and CIFAR-10. This experiment is intended as an architectural generalization check: the Transformer backbone differs substantially from convolutional feature extractors, while the final MLP classification head still induces exploitable inter-layer gradient correlations. We use the same FL and auxiliary data settings as in Section~\ref{sec:impl_details}, and additionally test a dummy data setting where the features and output probabilities are estimated from random Gaussian inputs.

\begin{table}[htbp]
  \centering
  \caption{GDBR label recovery performance on SmallViT across three datasets and four batch sizes (16, 32, 64, 128) using auxiliary data and dummy data. Results are reported as mean $\pm$ standard deviation.}
  \label{tab:vit_batch_size}

  \begin{threeparttable}
    \setlength{\tabcolsep}{4.8pt}
    \renewcommand{\arraystretch}{1.35}
    \begin{adjustbox}{width=0.96\columnwidth}
      \centering
      \begin{tabular}{cccccc}
        \specialrule{0.85pt}{2.5pt}{0.35pt}
        \multirow{2}{*}{\textbf{Dataset}} & \multirow{2}{*}{\textbf{Batch}} & \multicolumn{2}{c}{\textbf{Auxiliary Data}} & \multicolumn{2}{c}{\textbf{Dummy Data}} \\ \cline{3-6}
        & & InsAcc $\uparrow$ & ClsAcc $\uparrow$ & InsAcc $\uparrow$ & ClsAcc $\uparrow$ \\
        \specialrule{0.75pt}{0.35pt}{0.35pt}
          \multirow{4}{*}{MNIST} & 16  & 68.6{\fontsize{6}{5}\selectfont $\pm$12.1} & 82.0{\fontsize{6}{5}\selectfont $\pm$11.2} & 55.1{\fontsize{6}{5}\selectfont $\pm$18.4} & 77.5{\fontsize{6}{5}\selectfont $\pm$15.8} \\
          & 32  & 67.0{\fontsize{6}{5}\selectfont $\pm$13.4} & 91.0{\fontsize{6}{5}\selectfont $\pm$9.9}  & 33.9{\fontsize{6}{5}\selectfont $\pm$17.3} & 67.5{\fontsize{6}{5}\selectfont $\pm$13.4} \\
          & 64  & 53.5{\fontsize{6}{5}\selectfont $\pm$20.7} & 82.0{\fontsize{6}{5}\selectfont $\pm$18.1} & 55.1{\fontsize{6}{5}\selectfont $\pm$14.5} & 83.5{\fontsize{6}{5}\selectfont $\pm$11.5} \\
          & 128 & 66.0{\fontsize{6}{5}\selectfont $\pm$19.6} & 92.0{\fontsize{6}{5}\selectfont $\pm$18.3} & 39.7{\fontsize{6}{5}\selectfont $\pm$20.0} & 71.5{\fontsize{6}{5}\selectfont $\pm$17.1} \\
          \specialrule{0.35pt}{0.35pt}{0.35pt}
          \multirow{4}{*}{SVHN} & 16  & 63.6{\fontsize{6}{5}\selectfont $\pm$14.5} & 79.0{\fontsize{6}{5}\selectfont $\pm$11.8} & 39.4{\fontsize{6}{5}\selectfont $\pm$16.2} & 68.0{\fontsize{6}{5}\selectfont $\pm$16.9} \\
          & 32  & 66.5{\fontsize{6}{5}\selectfont $\pm$15.9} & 86.5{\fontsize{6}{5}\selectfont $\pm$11.9} & 41.4{\fontsize{6}{5}\selectfont $\pm$16.0} & 69.5{\fontsize{6}{5}\selectfont $\pm$13.6} \\
          & 64  & 64.0{\fontsize{6}{5}\selectfont $\pm$12.9} & 81.5{\fontsize{6}{5}\selectfont $\pm$9.1}  & 31.8{\fontsize{6}{5}\selectfont $\pm$16.0} & 64.5{\fontsize{6}{5}\selectfont $\pm$16.3} \\
          & 128 & 70.0{\fontsize{6}{5}\selectfont $\pm$16.3} & 91.5{\fontsize{6}{5}\selectfont $\pm$11.9} & 33.5{\fontsize{6}{5}\selectfont $\pm$14.4} & 61.5{\fontsize{6}{5}\selectfont $\pm$14.6} \\
          \specialrule{0.35pt}{0.35pt}{0.35pt}
          \multirow{4}{*}{CIFAR-10} & 16  & 63.6{\fontsize{6}{5}\selectfont $\pm$20.9} & 81.0{\fontsize{6}{5}\selectfont $\pm$13.0} & 39.6{\fontsize{6}{5}\selectfont $\pm$19.2} & 69.0{\fontsize{6}{5}\selectfont $\pm$16.1} \\
          & 32  & 59.5{\fontsize{6}{5}\selectfont $\pm$21.9} & 81.5{\fontsize{6}{5}\selectfont $\pm$16.2} & 32.8{\fontsize{6}{5}\selectfont $\pm$13.4} & 64.0{\fontsize{6}{5}\selectfont $\pm$12.4} \\
          & 64  & 65.2{\fontsize{6}{5}\selectfont $\pm$13.8} & 89.5{\fontsize{6}{5}\selectfont $\pm$12.8} & 38.6{\fontsize{6}{5}\selectfont $\pm$21.3} & 68.0{\fontsize{6}{5}\selectfont $\pm$16.6} \\
          & 128 & 67.3{\fontsize{6}{5}\selectfont $\pm$11.6} & 93.5{\fontsize{6}{5}\selectfont $\pm$11.5} & 34.0{\fontsize{6}{5}\selectfont $\pm$14.2} & 69.5{\fontsize{6}{5}\selectfont $\pm$15.3} \\
        \specialrule{0.85pt}{0.35pt}{0pt}
      \end{tabular}
    \end{adjustbox}
  \end{threeparttable}
\end{table}

As shown in Table~\ref{tab:vit_batch_size}, GDBR remains effective on transformer-based vision neural networks. With auxiliary data, GDBR achieves around 53-70\% InsAcc and 79-94\% ClsAcc across MNIST, SVHN, and CIFAR-10, despite the substantially different Transformer backbone. The dummy-data results are weaker, especially for instance-level recovery, but still reveal non-trivial class-level leakage across all three datasets. These findings indicate that the privacy risk exposed by GDBR is not confined to CNN-style feature extractors; transformer-based vision models with MLP classification heads can also leak label information through partially shared gradients.

\subsection{Performance under Local Updates}
\label{sec:local_updates}

Beyond the standard FedSGD setting, we further evaluate GDBR's efficacy under the FedAvg protocol, which allows clients to perform multiple local optimization steps before communicating updates. We test a local-update configuration where the victim client executes 1, 2, or 4 local SGD steps before sharing model updates. In this scenario, GDBR aims to reconstruct the complete set of labels utilized throughout all local updates from the shared updates. This experiment, conducted using auxiliary data, is intended to assess whether multi-step local updates can mitigate the proposed attack.

\begin{table}[htbp]
  \centering
  \caption{GDBR label recovery performance under local multi-step updates. Three local learning rates are tested for each number of local steps across three datasets and models. Results are reported as mean InsAcc / ClsAcc in percentage.}
  \label{tab:local_updates}

  \begin{threeparttable}
    \setlength{\tabcolsep}{4.8pt}
    \renewcommand{\arraystretch}{1.35}
    \begin{adjustbox}{width=0.96\columnwidth}
      \centering
      \begin{tabular}{cccccc}
        \specialrule{0.85pt}{2.5pt}{0.35pt}
        \multirow{2}{*}{\textbf{Dataset}} & \multirow{2}{*}{\textbf{Model}} & \multirow{2}{*}{\textbf{Steps}} & \multicolumn{3}{c}{\textbf{Local Learning Rate}} \\ \cline{4-6}
        & & & 0.001 & 0.002 & 0.005 \\
        \specialrule{0.75pt}{0.35pt}{0.35pt}
        \multirow{3}{*}{MNIST} & \multirow{3}{*}{LeNet} & 1 & 80.9 / 91.5 & 82.7 / 91.0 & 82.7 / 90.5 \\
        & & 2 & 66.6 / 96.5 & 64.5 / 98.0 & 66.7 / 97.5 \\
        & & 4 & 61.9 / 97.5 & 66.0 / 97.0 & 66.4 / 98.5 \\
        \specialrule{0.35pt}{0.35pt}{0.35pt}
        \multirow{3}{*}{SVHN} & \multirow{3}{*}{AlexNet} & 1 & 68.8 / 92.0 & 71.6 / 93.0 & 73.3 / 94.0 \\
        & & 2 & 55.8 / 84.0 & 49.9 / 86.0 & 16.7 / 84.5 \\
        & & 4 & 26.0 / 84.5 & 10.0 / 79.0 & 10.6 / 64.5 \\
        \specialrule{0.35pt}{0.35pt}{0.35pt}
        \multirow{3}{*}{CIFAR-10} & \multirow{3}{*}{VGG11} & 1 & 56.9 / 83.0 & 60.4 / 86.0 & 60.5 / 85.0 \\
        & & 2 & 69.0 / 96.5 & 69.0 / 97.5 & 61.0 / 85.0 \\
        & & 4 & 62.4 / 97.5 & 65.4 / 96.5 & 28.2 / 83.0 \\
        \specialrule{0.85pt}{0.35pt}{0pt}
      \end{tabular}
    \end{adjustbox}
  \end{threeparttable}
\end{table}

As shown in Table~\ref{tab:local_updates}, local multi-step updates do not eliminate label leakage. On MNIST/LeNet, GDBR preserves high class-level recovery across all local steps and learning rates, with 90.5-98.5\% ClsAcc, while InsAcc decreases from around 81-83\% at one local step to 62-67\% after multiple steps. On CIFAR-10/VGG11, GDBR remains strong for moderate local-update settings, reaching 69.0\% InsAcc and 97.5\% ClsAcc at two local steps, although a larger learning rate with four steps degrades instance-level recovery. On SVHN/AlexNet, instance-level recovery is more sensitive as local updates increase, but class-level leakage remains substantial. Overall, FedAvg-style local updates can reduce exact instance-level recovery in some aggressive settings, but they do not remove the class-level label leakage exploited by GDBR.

\subsection{Performance Against Defense Mechanisms}

We finally evaluate the robustness of GDBR against additional defense mechanisms. We consider two representative defense mechanisms widely discussed in the federated learning literature: \emph{gradient pruning} and \emph{noise perturbation}. These defenses aim to reduce the information leakage from shared gradients by either sparsifying them or injecting random noise. Specifically, we assess GDBR's performance under varying pruning thresholds (e.g., 0.4, 0.6, 0.8, 0.9) and noise scales (e.g., 0.01, 0.05, 0.1, 0.2) to simulate different levels of gradient obfuscation.

For gradient pruning, elements of the gradient whose magnitude falls below a given threshold are set to zero, thereby removing weak signals that could be exploited for label inference. For noise perturbation, Gaussian noise is added to the gradients, with the standard deviation scaled according to the specified noise level, simulating differential privacy-like randomization. 

As illustrated in Fig.~\ref{fig:grad_defense}, GDBR exhibits substantial resilience under moderate defense configurations. However, its performance undergoes a significant decline when confronted with aggressive pruning (e.g., $\geq 0.9$) or intensive noise injection (e.g., $\geq 0.2$). Notably, the ResNet-18 model trained on CIFAR-100 appears disproportionately susceptible to additive noise compared to other architectures. This heightened sensitivity is primarily attributed to the accumulation and amplification of stochastic perturbations during the derivation of gradients for $\mathbf{a}_{k}$ in the penultimate layer. According to
$
\nabla\mathbf{a}_k = \langle \nabla\mathbf{W}_k, \mathbf{W}_k \rangle_\text{F} \oslash \mathbf{a}_k,
$
noise present in $\nabla\mathbf{W}_k$ is magnified by the Frobenius inner product, thereby increasing its disruptive impact on the recovered features.

Overall, these findings indicate that while GDBR is resilient under mild gradient obfuscation, stronger defense mechanisms can meaningfully hinder its label recovery capability. This underscores the necessity for designing more robust and adaptive privacy-preserving strategies in gradient-sharing frameworks.

\begin{figure}[!t]
  \centering
  \begin{subfigure}[b]{0.232\textwidth}
    \centering
    \includegraphics[width=\textwidth]{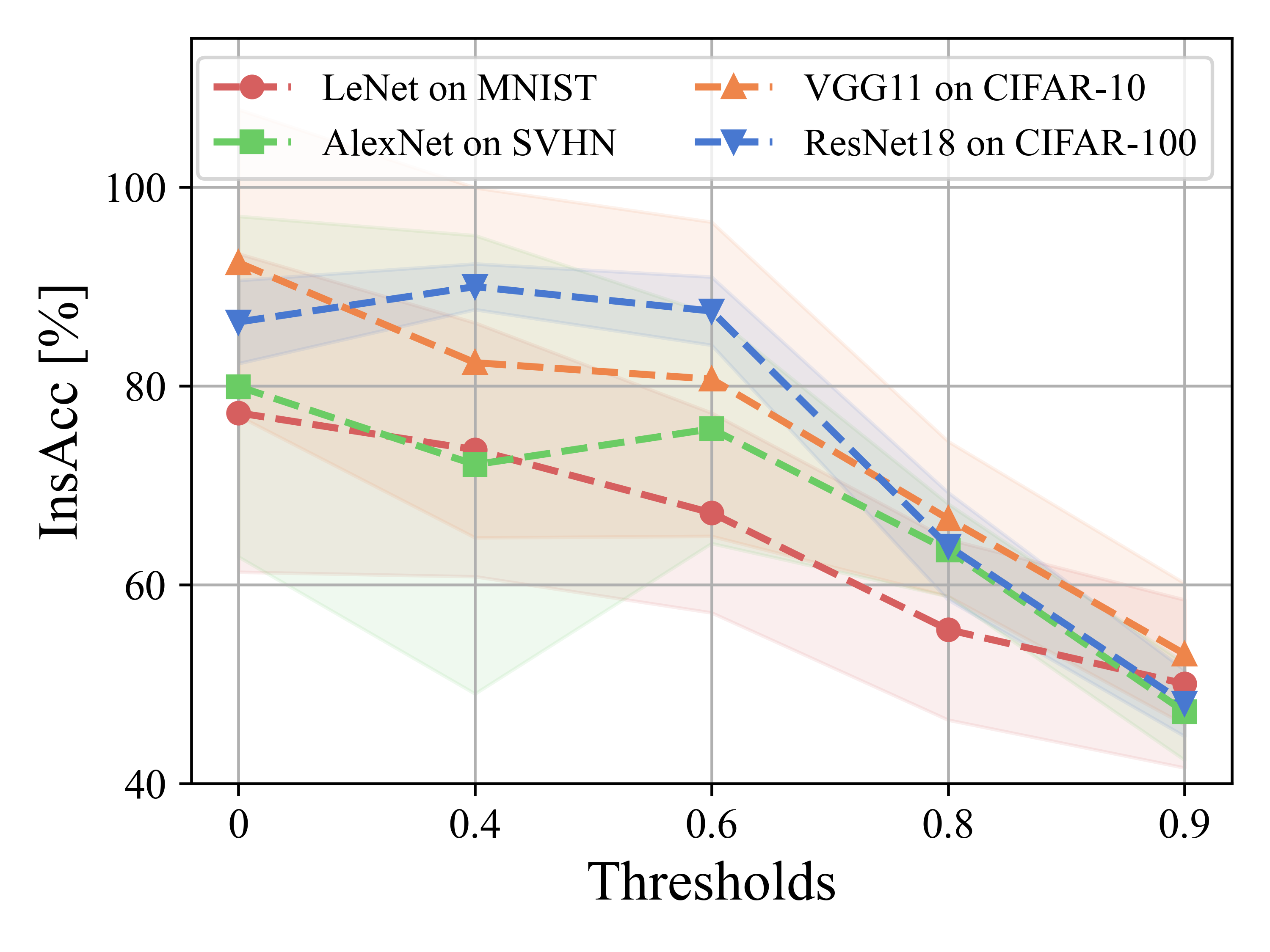}
    \caption{Performance under pruning.}
  \end{subfigure}
  \hfill
  \begin{subfigure}[b]{0.232\textwidth}
    \centering
    \includegraphics[width=\textwidth]{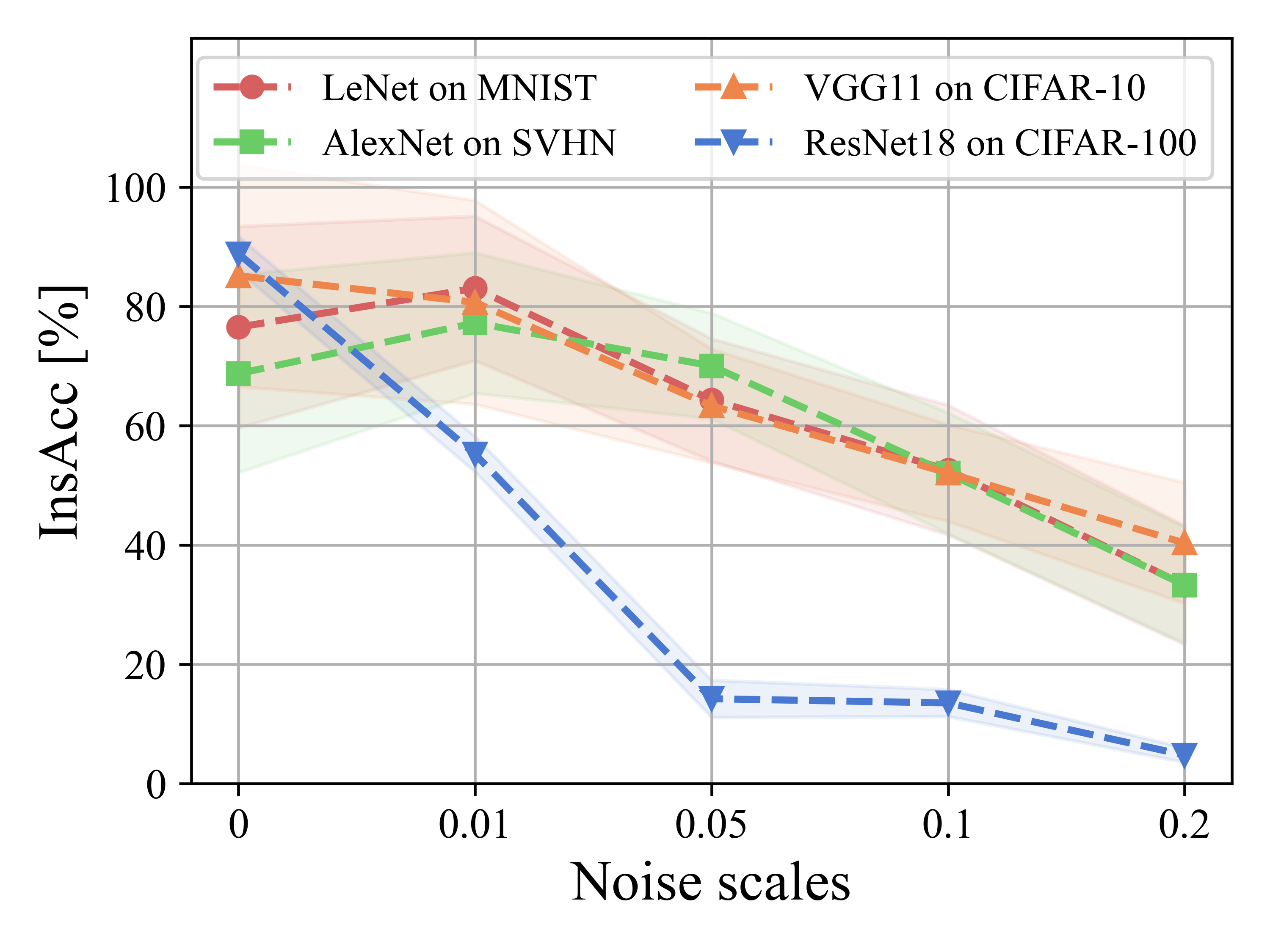}
    \caption{Performance under noising.}
  \end{subfigure}
  \caption{Comparison of gradient pruning and gradient perturbation on InsAcc of GDBR across different models and datasets. The x-axis denotes the pruning thresholds or noise scales, while the y-axis shows the InsAcc. Results are averaged over 20 trials; shaded area shows standard deviation.}
  \label{fig:grad_defense}
\end{figure}

\section{Conclusion}

In this paper, we propose GDBR, a label recovery attack that restores the label distributions of client's batch training data from partial encrypted gradients shared in federated learning (FL). By leveraging the relationships between shared gradients and model parameters, GDBR formulates an effective approach for label restoration. Extensive experiments show that GDBR can effectively and accurately recover the classwise batch labels across various FL settings. In future work, we aim to extend our approach to support a wider range of activation functions, and develop efficient data reconstruction methods based on our theoretical insights. Additionally, we plan to explore and design robust defense mechanisms against GDBR to enhance the security of FL systems.

\section*{Acknowledgment}

This research is partially supported by the RGC Early Career Scheme (Project \#27211524) and Croucher Start-up Allowance (Project \#2499102828). Any opinions, findings, or conclusions expressed in this material are those of the authors and do not necessarily reflect the views of RGC and the Croucher Foundation.

\newpage
\bibliographystyle{IEEEtran}
\bibliography{reference}

\begin{thebibliography}{10}
\providecommand{\url}[1]{#1}
\csname url@samestyle\endcsname
\providecommand{\newblock}{\relax}
\providecommand{\bibinfo}[2]{#2}
\providecommand{\BIBentrySTDinterwordspacing}{\spaceskip=0pt\relax}
\providecommand{\BIBentryALTinterwordstretchfactor}{4}
\providecommand{\BIBentryALTinterwordspacing}{\spaceskip=\fontdimen2\font plus
\BIBentryALTinterwordstretchfactor\fontdimen3\font minus \fontdimen4\font\relax}
\providecommand{\BIBforeignlanguage}[2]{{%
\expandafter\ifx\csname l@#1\endcsname\relax
\typeout{** WARNING: IEEEtran.bst: No hyphenation pattern has been}%
\typeout{** loaded for the language `#1'. Using the pattern for}%
\typeout{** the default language instead.}%
\else
\language=\csname l@#1\endcsname
\fi
#2}}
\providecommand{\BIBdecl}{\relax}
\BIBdecl

\bibitem{konevcny2016federated}
J.~Kone{\v{c}}n{\`y}, H.~B. McMahan, F.~X. Yu, P.~Richt{\'a}rik, A.~T. Suresh, and D.~Bacon, ``Federated learning: Strategies for improving communication efficiency,'' in \emph{NeurIPS Workshop on Private Multi-Party Machine Learning (PMPML)}, 2016.

\bibitem{mcmahan2017communication}
B.~McMahan, E.~Moore, D.~Ramage, S.~Hampson, and B.~A. y~Arcas, ``Communication-efficient learning of deep networks from decentralized data,'' in \emph{Proceedings of the 20th International Conference on Artificial Intelligence and Statistics (AISTATS)}, 2017, pp. 1273--1282.

\bibitem{bonawitz2019towards}
K.~Bonawitz, H.~Eichner, W.~Grieskamp, D.~Huba, A.~Ingerman, V.~Ivanov, C.~Kiddon, J.~Kone{\v{c}}n{\`y}, S.~Mazzocchi, B.~McMahan \emph{et~al.}, ``Towards federated learning at scale: System design,'' in \emph{Proceedings of the 2nd Machine Learning and Systems (MLSys)}, 2019, pp. 374--388.

\bibitem{wei2020framework}
W.~Wei, L.~Liu, M.~Loper, K.-H. Chow, M.~E. Gursoy, S.~Truex, and Y.~Wu, ``A framework for evaluating gradient leakage attacks in federated learning,'' in \emph{Proceedings of the 25th European Symposium on Research in Computer Security (ESORICS)}, 2020, pp. 545--566.

\bibitem{zhang2023survey}
R.~Zhang, S.~Guo, J.~Wang, X.~Xie, and D.~Tao, ``A survey on gradient inversion: Attacks, defenses and future directions,'' in \emph{Proceedings of the 31st International Joint Conference on Artificial Intelligence (IJCAI)}, 2023, pp. 5678--5685.

\bibitem{zhao2020idlg}
B.~Zhao, K.~R. Mopuri, and H.~Bilen, ``{iDLG}: Improved deep leakage from gradients,'' \emph{arXiv preprint arXiv:2001.02610}, 2020.

\bibitem{dang2021revealing}
T.~Dang, O.~Thakkar, S.~Ramaswamy, R.~Mathews, P.~Chin, and F.~Beaufays, ``Revealing and protecting labels in distributed training,'' in \emph{Proceedings of the 35th Conference on Neural Information Processing Systems (NeurIPS)}, 2021, pp. 1727--1738.

\bibitem{yin2021see}
H.~Yin, A.~Mallya, A.~Vahdat, J.~M. Alvarez, J.~Kautz, and P.~Molchanov, ``See through gradients: Image batch recovery via gradinversion,'' in \emph{Proceedings of the 2021 IEEE/CVF Conference on Computer Vision and Pattern Recognition (CVPR)}, 2021, pp. 16\,337--16\,346.

\bibitem{geng2021towards}
J.~Geng, Y.~Mou, F.~Li, Q.~Li, O.~Beyan, S.~Decker, and C.~Rong, ``Towards general deep leakage in federated learning,'' in \emph{AAAI Workshop on Trustable, Verifiable and Auditable Federated Learning (FL-AAAI-22)}, 2022.

\bibitem{wainakh2022user}
A.~Wainakh, F.~Ventola, T.~M{\"u}{\ss}ig, J.~Keim, C.~G. Cordero, E.~Zimmer, T.~Grube, K.~Kersting, and M.~M{\"u}hlh{\"a}user, ``User-level label leakage from gradients in federated learning,'' in \emph{Proceedings of the 22nd Privacy Enhancing Technologies Symposium (PETS)}, 2022, pp. 227--244.

\bibitem{ma2023instancewise}
K.~Ma, Y.~Sun, J.~Cui, D.~Li, Z.~Guan, and J.~Liu, ``Instance-wise batch label restoration via gradients in federated learning,'' in \emph{Proceedings of the 11th International Conference on Learning Representations (ICLR)}, 2023.

\bibitem{wang2024graddiff}
X.~Wang, L.~Wu, and Z.~Guan, ``{GradDiff}: Gradient-based membership inference attacks against federated distillation with differential comparison,'' \emph{Information Sciences}, vol. 658, p. 120068, 2024.

\bibitem{mukhtiar2025fairness}
N.~Mukhtiar, A.~Mahmood, and Q.~Z. Sheng, ``Fairness in federated learning: Trends, challenges, and opportunities,'' \emph{Advanced Intelligent Systems}, p. 2400836, 2025.

\bibitem{zhang2020batchcrypt}
C.~Zhang, S.~Li, J.~Xia, W.~Wang, F.~Yan, and Y.~Liu, ``{BatchCrypt}: Efficient homomorphic encryption for cross-silo federated learning,'' in \emph{2020 USENIX annual technical conference (USENIX ATC)}, 2020, pp. 493--506.

\bibitem{fang2021privacy}
H.~Fang and Q.~Qian, ``Privacy preserving machine learning with homomorphic encryption and federated learning,'' \emph{Future Internet}, vol.~13, no.~4, p.~94, 2021.

\bibitem{zhu2019deep}
L.~Zhu, Z.~Liu, and S.~Han, ``Deep leakage from gradients,'' in \emph{Advances in Neural Information Processing Systems (NeurIPS)}, 2019.

\bibitem{geiping2020inverting}
J.~Geiping, H.~Bauermeister, H.~Dr{\"o}ge, and M.~Moeller, ``Inverting gradients - how easy is it to break privacy in federated learning?'' in \emph{Proceedings of the 34th Conference on Neural Information Processing Systems (NeurIPS)}, 2020, pp. 16\,937--16\,947.

\bibitem{zhu2021rgap}
J.~Zhu and M.~Blaschko, ``{R-GAP}: Recursive gradient attack on privacy,'' in \emph{Proceedings of the 9th International Conference on Learning Representations (ICLR)}, 2021.

\bibitem{gupta2022recovering}
S.~Gupta, Y.~Huang, Z.~Zhong, T.~Gao, K.~Li, and D.~Chen, ``Recovering private text in federated learning of language models,'' \emph{Proceedings of the 36th Conference on Neural Information Processing Systems (NeurIPS)}, pp. 8130--8143, 2022.

\bibitem{balunovic2022lamp}
M.~Balunovic, D.~Dimitrov, N.~Jovanovi{\'c}, and M.~Vechev, ``{LAMP}: Extracting text from gradients with language model priors,'' in \emph{Proceedings of the 36th Conference on Neural Information Processing Systems (NeurIPS)}, 2022, pp. 7641--7654.

\bibitem{dang2022method}
T.~Dang, O.~Thakkar, S.~Ramaswamy, R.~Mathews, P.~Chin, and F.~Beaufays, ``A method to reveal speaker identity in distributed asr training, and how to counter it,'' in \emph{Proceedings of the 2022 IEEE International Conference on Acoustics, Speech and Signal Processing (ICASSP)}, 2022, pp. 4338--4342.

\bibitem{jin2023fedml}
W.~Jin, Y.~Yao, S.~Han, J.~Gu, C.~Joe-Wong, S.~Ravi, S.~Avestimehr, and C.~He, ``{FedML-HE}: An efficient homomorphic-encryption-based privacy-preserving federated learning system,'' \emph{arXiv preprint arXiv:2303.10837}, 2023.

\bibitem{hu2024maskcrypt}
C.~Hu and B.~Li, ``{Maskcrypt}: Federated learning with selective homomorphic encryption,'' \emph{IEEE Transactions on Dependable and Secure Computing}, vol.~22, no.~1, pp. 221--233, 2024.

\bibitem{liu2025shelora}
J.~Liu, L.~Yan, B.~Li, L.~Yu, and C.~Shen, ``{SHE-LoRA}: Selective homomorphic encryption for federated tuning with heterogeneous lora,'' \emph{arXiv preprint arXiv:2505.21051}, 2025.

\bibitem{li2025sensecrypt}
B.~Li, L.~Yan, J.~Han, J.~Liu, and L.~Yu, ``{SenseCrypt}: Sensitivity-guided selective homomorphic encryption for joint federated learning in cross-device scenarios,'' \emph{arXiv preprint arXiv:2508.04100}, 2025.

\bibitem{lecun1998gradient}
Y.~LeCun, L.~Bottou, Y.~Bengio, and P.~Haffner, ``Gradient-based learning applied to document recognition,'' \emph{Proceedings of the IEEE}, vol.~86, no.~11, pp. 2278--2324, 1998.

\bibitem{netzer2011reading}
Y.~Netzer, T.~Wang, A.~Coates, A.~Bissacco, B.~Wu, A.~Y. Ng \emph{et~al.}, ``Reading digits in natural images with unsupervised feature learning,'' in \emph{NeurIPS Workshop on Deep Learning and Unsupervised Feature Learning (DLUFL)}, 2011.

\bibitem{krizhevsky2009learning}
A.~Krizhevsky, G.~Hinton \emph{et~al.}, ``Learning multiple layers of features from tiny images,'' \emph{Technical Report, University of Toronto}, 2009.

\bibitem{deng2009imagenet}
J.~Deng, W.~Dong, R.~Socher, L.-J. Li, K.~Li, and L.~Fei-Fei, ``{ImageNet}: A large-scale hierarchical image database,'' in \emph{Proceedings of the 2009 IEEE Conference on Computer Vision and Pattern Recognition (CVPR)}, 2009, pp. 248--255.

\bibitem{rumelhart1987learning}
D.~E. Rumelhart, G.~E. Hinton, and R.~J. Williams, ``Learning internal representations by error propagation,'' in \emph{Parallel Distributed Processing: Explorations in the Microstructure of Cognition: Foundations}.\hskip 1em plus 0.5em minus 0.4em\relax MIT Press, 1987, pp. 318--362.

\bibitem{krizhevsky2012imagenet}
A.~Krizhevsky, I.~Sutskever, and G.~E. Hinton, ``{ImageNet} classification with deep convolutional neural networks,'' \emph{Proceedings of the 26th Conference on Neural Information Processing Systems (NeurIPS)}, 2012.

\bibitem{simonyan2014very}
K.~Simonyan and A.~Zisserman, ``Very deep convolutional networks for large-scale image recognition,'' \emph{arXiv preprint arXiv:1409.1556}, 2014.

\bibitem{he2016deep}
K.~He, X.~Zhang, S.~Ren, and J.~Sun, ``Deep residual learning for image recognition,'' in \emph{Proceedings of the 2016 IEEE Conference on Computer Vision and Pattern Recognition (CVPR)}, 2016, pp. 770--778.

\bibitem{dosovitskiy2021image}
A.~Dosovitskiy, L.~Beyer, A.~Kolesnikov, D.~Weissenborn, X.~Zhai, T.~Unterthiner, M.~Dehghani, M.~Minderer, G.~Heigold, S.~Gelly, J.~Uszkoreit, and N.~Houlsby, ``An image is worth 16x16 words: Transformers for image recognition at scale,'' in \emph{Proceedings of the 9th International Conference on Learning Representations (ICLR)}, 2021.

\bibitem{paszke2019pytorch}
A.~Paszke, S.~Gross, F.~Massa, A.~Lerer, J.~Bradbury, G.~Chanan, T.~Killeen, Z.~Lin, N.~Gimelshein, L.~Antiga \emph{et~al.}, ``{PyTorch}: An imperative style, high-performance deep learning library,'' \emph{Proceedings of the 33rd Conference on Neural Information Processing Systems (NeurIPS)}, 2019.

\bibitem{zhang2024posterior}
R.~Zhang, S.~Guo, and P.~Li, ``Posterior probability-based label recovery attack in federated learning,'' in \emph{ICLR Workshop on Privacy Regulation and Protection in Machine Learning (PML)}, 2024.

\bibitem{he2015delving}
K.~He, X.~Zhang, S.~Ren, and J.~Sun, ``Delving deep into rectifiers: Surpassing human-level performance on imagenet classification,'' in \emph{Proceedings of the 2015 IEEE International Conference on Computer Vision (ICCV)}, 2015, pp. 1026--1034.

\end{thebibliography}

\end{document}